\documentclass[conference]{IEEEtran}
\IEEEoverridecommandlockouts
\AtBeginDocument{%
  \setlength{\textheight}{660pt}%
}

\usepackage{cite}
\usepackage{amsmath,amssymb,amsfonts}
\usepackage{amsthm}
\usepackage{algorithm}
\usepackage{algorithmic}
\usepackage{graphicx}
\usepackage{textcomp}
\usepackage{xcolor}
\usepackage{tikz}
\usetikzlibrary{arrows.meta}

\usepackage[
  pdfauthor={Wentao Zhang; Wentao Mo},
  pdftitle={Real-Time Hard Peak Age-of-Information Safety with No-Regret Learning},
  pdfsubject={},
  pdfkeywords={Age of Information, online convex optimization, real-time scheduling, safety-critical systems, adversarial channels}
]{hyperref}
\hypersetup{hidelinks}

\theoremstyle{plain}
\newtheorem{theorem}{Theorem}
\newtheorem{lemma}{Lemma}
\newtheorem{proposition}{Proposition}
\newtheorem{corollary}{Corollary}

\theoremstyle{definition}
\newtheorem{assumption}{Assumption}
\newtheorem{definition}{Definition}

\theoremstyle{remark}
\newtheorem*{remark}{Remark}
\def\BibTeX{{\rm B\kern-.05em{\sc i\kern-.025em b}\kern-.08em
    T\kern-.1667em\lower.7ex\hbox{E}\kern-.125emX}}
\begin{document}
\bstctlcite{IEEEtranBSTcontrol}

\title{Real-Time Hard Peak Age-of-Information Safety \\ with No-Regret Learning}

\author{%
\IEEEauthorblockN{Wentao Zhang and Wentao Mo\IEEEauthorrefmark{1}%
\thanks{\IEEEauthorrefmark{1} Corresponding author.}}
\IEEEauthorblockA{Tsinghua Shenzhen International Graduate School, Tsinghua University\\
Shenzhen, China\\
\{zhang-wt24,mow10\}@mails.tsinghua.edu.cn}
}

\maketitle

\begin{abstract}
Safety-critical IoT systems such as industrial closed-loop control, V2X coordination, and remote teleoperation require every sensor's peak Age of Information (peak AoI, also abbreviated PAoI) to stay below a hard per-slot deadline, not merely an average bound. Existing approaches meet this requirement only under restrictive assumptions: stochastic channels for Whittle-index AoI, simulator rollouts for deep reinforcement learning, or sublinear cumulative violation for long-term constrained online convex optimization. \textcolor{black}{Under adversarial coefficients, OCO-PAoI-Hard guarantees zero per-slot violation of the modeled AoI state under one-step viability and $O(\sqrt{T})$ regret against any static safe comparator; packet-level safety requires stronger service assumptions.} Our key observation is that the fractional peak-AoI deadline collapses exactly to an affine half-space constraint on the \textcolor{black}{resource-allocation vector}, turning hard real-time scheduling into time-varying constrained online convex optimization over a polyhedral safe set. A strictly causal proposal-shield-update loop enforces feasibility through one Euclidean projection per slot, the gradient step preserves no-regret behaviour, and the classical virtual queue is reduced to an a-posteriori certificate. We establish closed-form static and dynamic regret bounds, a matching $\Omega(\sqrt{T})$ minimax lower bound, a margin-safe variant against execution noise, and a deadline-induced competitive ratio. \textcolor{black}{On a four-sensor adversarial fluid-model trap channel, OCO-PAoI-Hard attains zero modeled-state deadline violations across all ten seeds} while four representative baselines miss between $1.65\%$ and $64.0\%$ of slots, and the empirical normalized regret stays below the theoretical envelope across two orders of magnitude in $T$.
\end{abstract}

\begingroup
\color{black}
\endgroup

\section{Introduction}
\label{sec:intro}

Safety-critical wireless cyber-physical systems, including industrial closed-loop control, V2X coordination, remote surgery, smart-grid protection, and immersive XR teleoperation, push real-time networking well beyond best-effort throughput. What determines correctness is not how often a sensor's data arrives, but how recent the controller's last fresh sample is, captured by the Age of Information (AoI) metric \cite{kaul2012aoi,sun2017aoisurvey}. The bottleneck is the worst-case AoI rather than its average: IEC~61850-5 assigns protection samples to transfer-time class TT6 ($\le 3\,\mathrm{ms}$), and 3GPP specifies a general URLLC packet-success probability of $1-10^{-5}$ for 32 bytes within $1\,\mathrm{ms}$ user-plane latency \cite{iec61850,tr38913}, so a single missed deadline trips a circuit breaker, derails a haptic-feedback loop, or destabilizes a V2X platoon. Multi-sensor IoT systems must therefore meet a per-sensor peak-AoI deadline at every slot while sharing a wireless link subject to fading, interference, and adversarial blockers. The question we address is whether sensors can be scheduled over an adversarial shared channel so that every sensor's peak AoI stays below its deadline at every slot, while a good policy is still learned online.

Three decades of work approach this question along progressively weaker assumptions, but none answers it completely. The classical Markovian line, including restless multi-armed bandits and the Whittle-index AoI scheduler, treats freshness as a soft cost and proves index-optimality under i.i.d.\ or Markovian channel statistics \cite{hsu2018whittle,maatouk2022whittle,tripathi2021whittle}, but the argument collapses once channels become non-stationary. Deep reinforcement learning replaces model-based analysis with end-to-end policy search \cite{leng2019drl,ceran2021drlaoi,abdelmagid2022drlsurvey}, delivering strong empirical reward but offering neither regret nor violation guarantees. A third line brings adversarial robustness through online competitive analysis \cite{banerjee2020adversarial,bhattacharjee2023competitive,bhuyan2022adv}, but the guarantee is on a long-run ratio rather than a per-slot per-sensor deadline. In parallel, drift-plus-penalty and long-term virtual-queue OCO methods \cite{yu2017online,jenatton2016ococonstraint,mahdavi2012trading} achieve simultaneous $O(\sqrt{T})$ regret and $o(T)$ cumulative constraint violation, yet explicitly allow individual slots to violate the deadline as long as the running total grows sublinearly. Even an algorithm with $O(1)$ cumulative violation can still miss a deadline in some slot, which under the weakly-hard $(m,k)$-firm semantics \cite{bernat2001weaklyhard,hammadeh2017mkfirm,sun2023mkfirm} is precisely the failure mode that propagates to control-loop instability. \textcolor{black}{No existing framework simultaneously offers zero per-slot violation of the modeled peak-AoI state, $O(\sqrt{T})$ no-regret adaptation, and robustness to adversarial coefficient traces.}

This paper proposes OCO-PAoI-Hard, a scheduling framework that closes all three gaps with a single causal algorithm. The starting point is a structural observation. Under the \textcolor{black}{modeled} AoI dynamics, the per-sensor peak-AoI deadline collapses exactly, without any slack, to an affine half-space constraint $C_t x \ge \theta_t$ on the \textcolor{black}{fractional resource allocation}, where $C_t$ encodes the realized channel, fresh-arrival, and common-observation pattern in slot $t$, and $\theta_t$ is a refresh threshold determined by the current AoI and deadline vectors. This reduction turns hard real-time AoI scheduling into time-varying constrained OCO over a polyhedral safe set $\mathcal{K}_t$. We then build a strictly causal proposal-shield-update loop: a gradient proposal $z_t$ is hardened into a feasible action $x_t$ by Euclidean projection onto $\mathcal{K}_t$, the executed action incurs loss $f_t(x_t)$, and the gradient update produces $z_{t+1}$. \textcolor{black}{The shield enforces the modeled-state deadline at every realized state--refresh pair for which $\mathcal K_t\ne\varnothing$; this non-emptiness is necessary for any current action to meet all modeled deadlines from that state.} The virtual queue is retained only as a zero-violation certificate. Table~\ref{tab:comparison} contrasts our guarantees with the four prior families. \textcolor{black}{On a four-sensor adversarial fluid-model trap channel, OCO-PAoI-Hard records zero modeled-state deadline violations} across all $10$ seeds (max numerical excess $\le 10^{-12}$), whereas vanilla OGD, drift-plus-penalty, greedy max-deficit, and round-robin baselines suffer per-slot miss rates of $41.3\%$, $1.65\%$, $64.0\%$, and $63.6\%$, respectively, and the empirical static regret normalized by $RG\sqrt{T}$ stays below $0.4$ across $T\in[200,2\!\times\!10^{4}]$, in agreement with the $O(\sqrt{T})$ theory.

\begin{table}[t]
\caption{Comparison with representative AoI and constrained-OCO scheduling frameworks. ``Hard PAoI'' is the strongest per-sensor peak-AoI guarantee provided, and ``Regret'' is the static-regret order over horizon $T$. ``Sim/stoch.'' denotes simulator-trained or stochastic-channel assumptions.}
\label{tab:comparison}
\centering
\renewcommand{\arraystretch}{1.18}
\setlength{\tabcolsep}{3.5pt}
\begin{tabular}{|l|c|c|c|}
\hline
Method & Channel & Hard PAoI & Regret\\
\hline
Whittle AoI \cite{hsu2018whittle,maatouk2022whittle} & i.i.d./Markov & soft cost & N/A\\
\hline
DRL AoI \cite{leng2019drl,ceran2021drlaoi} & sim/stoch. & empirical & N/A\\
\hline
Adv.\ AoI \cite{banerjee2020adversarial,bhattacharjee2023competitive} & adversarial & time-avg.\ ratio & N/A\\
\hline
LT-OCO \cite{yu2017online,jenatton2016ococonstraint} & adversarial & cumul.\ $o(T)$ & $O(\sqrt{T})$\\
\hline
Ours & adversarial & \textcolor{black}{fluid/expected, per slot} & $O(\sqrt{T})$\\
\hline
\end{tabular}
\end{table}

Our contributions are as follows.
\begin{itemize}
\item \textbf{AoI-to-safety reduction and algorithm.} We prove that \textcolor{black}{modeled} peak-AoI deadlines reduce exactly to affine half-space constraints on the fractional \textcolor{black}{resource allocation}, and use this reduction to design OCO-PAoI-Hard, a strictly causal proposal-shield-update algorithm that, to our knowledge, is the first to deliver $O(\sqrt{T})$ regret and \textcolor{black}{zero modeled-state violation} under fully adversarial channel, arrival, and common-observation sequences, \textcolor{black}{conditional on realized one-step viability}.
\item \textbf{Complete theoretical analysis with matching lower bound.} We establish a $\frac{3}{2}RG\sqrt{T}$ static-regret bound, a path-length dynamic-regret bound $RG\sqrt{T}+G\sqrt{T}\,P_T$, a tight $\Omega(\sqrt{T})$ minimax lower bound, a margin perturbation theorem $\xi\ge L\,\varepsilon$ that recovers \textcolor{black}{modeled-state safety} under approximate projection and execution noise, a deadline-induced competitive ratio $\rho_D=(\sum_i w_iD_i)/W$, a regret-refined trace-wise ratio $\alpha_T+RG/(W\sqrt{T})$, and a fractional-to-integral interface delineating when fluid guarantees lift to packet-level execution.
\item \textbf{Empirical validation on adversarial channels.} We validate every theoretical prediction on a multi-sensor adversarial shared-channel testbed: \textcolor{black}{zero modeled fluid-state deadline violations} for our shielded scheduler against four representative baselines, a normalized regret ratio below the theoretical envelope across two orders of magnitude in $T$, and a clean $\xi\ge L\,\varepsilon$ phase boundary on the margin and noise plane.
\end{itemize}

\section{Related Work}
\label{sec:related}

Our work bridges three lines of inquiry: AoI scheduling for wireless status updating, constrained online convex optimization, and real-time scheduling with safety guarantees.

\subsection{AoI Scheduling for Wireless Status Updates}
\label{sec:rw-aoi}

Since Kaul et al.\ \cite{kaul2012aoi} formalized AoI as a freshness metric, an extensive algorithmic toolbox has been built (see the survey \cite{sun2017aoisurvey}). For a single source-channel pair, Sun et al.\ \cite{sun2019aoi} derive optimal samplers under general service distributions, and Bedewy et al.\ \cite{bedewy2021aoi} extend them to multi-source preemptive systems with random arrivals and parallel servers. In the multi-source shared-channel regime closest to our setting, the Whittle-index family \cite{hsu2018whittle,maatouk2022whittle,tripathi2021whittle} reduces an otherwise intractable restless bandit to per-source threshold rules under i.i.d.\ or Markovian channels, and Kadota et al.\ \cite{kadota2018scheduling} prove matching lower bounds and an asymptotically optimal stationary randomized policy. Two more recent threads relax these assumptions: online competitive analysis \cite{banerjee2020adversarial,bhattacharjee2023competitive,bhuyan2022adv} removes channel stationarity and targets a constant-competitive ratio for the time-averaged AoI, while data-driven schedulers replace analytical optima with deep-RL policies trained on simulator rollouts \cite{leng2019drl,ceran2021drlaoi,abdelmagid2022drlsurvey}. All these directions treat AoI as a stochastic cost, a long-run competitive ratio, or a learning target without provable feasibility; \textcolor{black}{we instead certify the per-slot invariant of the modeled AoI state, under the scope stated above.}

\subsection{Constrained Online Convex Optimization}
\label{sec:rw-oco}

The OCO framework \cite{zinkevich2003ogd,hazan2016oco} provides $O(\sqrt{T})$ regret for adversarial convex losses, an information-theoretically optimal rate even without constraints. Constrained extensions \cite{mahdavi2012trading,jenatton2016ococonstraint,yu2017online,yi2021oco} achieve $O(\sqrt{T})$ regret with sublinear cumulative violation under stochastic constraints, and dynamic-environment work \cite{chen2017oco,cao2019oco,valls2020online} handles drifting feasible sets and switching costs. \textcolor{black}{Prediction-assisted dynamic-regret methods exploit untrusted per-round-minimizer predictions while retaining the optimal prediction-free rate under adversarial predictions~\cite{zhang2026dynamic}. Complementary results give logarithmic dependence on the number of adversarial constraints in each constraint's cumulative-violation bound~\cite{zhang2026multiconstraint} and, in separate settings, noise-adaptive high-probability regret under sub-Gaussian gradients and simultaneous high-probability regret/long-run-violation guarantees under stochastic constraints~\cite{zhang2026noiseadaptivehighprobabilityregretbounds}.} Drift-plus-penalty and long-term virtual-queue methods absorb constraint violations into a sublinear cumulative budget, but explicitly allow individual slots to violate the deadline as long as the running total grows sublinearly; even an algorithm with $O(1)$ cumulative violation can still miss a deadline in some slot. Our algorithm instead enforces feasibility at every slot through Euclidean projection onto the current-slot $\mathcal{K}_t$, and the virtual queue is reduced from a primary safety mechanism to an a-posteriori zero-violation certificate, separating shield from accounting and yielding \textcolor{black}{per-slot modeled-state rather than long-run guarantees}.

\subsection{Hard Real-Time and Safe-Learning Schedulers}
\label{sec:rw-realtime}

Classical hard real-time theory \cite{davisburns2011realtime} achieves per-task deadlines through schedulability analysis with known service rates, while the weakly-hard $(m,k)$-firm model \cite{bernat2001weaklyhard,hammadeh2017mkfirm,sun2023mkfirm}, widely used in the hard real-time community, relaxes this to ``at most $m$ misses in any $k$-window,'' yet ultimately relies on a base scheduler that does enforce per-slot deadlines. URLLC \cite{bennis2018urllc,tr38913} delivers sub-millisecond reliability through resource reservation and packet replication but does not embed an online learning loop. The safe-RL literature provides complementary tools: trust-region projections under expected-cost constraints \cite{achiam2017cpo}, reactive shields filtering unsafe actions in finite-state MDPs \cite{alshiekh2018shielding}, and zero-duality-gap results for constrained MDPs \cite{paternain2019safe}. These works depend on offline schedulability tests, on stochastic constrained MDPs with expected-violation guarantees, or on shields synthesized from pre-specified safety automata. Our framework instead derives the safety shield analytically from the AoI dynamics via the affine reduction $C_t x \ge \theta_t$, instantiates it from observations available at the start of each slot, and composes it with an online gradient update so that \textcolor{black}{zero modeled-state deadline violations and $O(\sqrt{T})$ regret hold simultaneously under adversarial traces, conditional on realized one-step viability}.

\section{Problem Setup}
\label{sec:problem}

We consider a single-cell, $N$-sensor IoT system that shares a wireless uplink to a common controller over a slotted horizon $t=1,\dots,T$. Vectors are columns; $\mathbf 1\in\mathbb R^N$ is the all-ones vector and $e_j$ the $j$-th canonical basis vector; $\|\cdot\|_2$ is the Euclidean norm, $\langle\cdot,\cdot\rangle$ the inner product, $a\odot b$ the Hadamard product, $[y]_+=\max\{y,0\}$, and $\Pi_{\mathcal C}(y)=\arg\min_{x\in\mathcal C}\|x-y\|_2$ the unique Euclidean projection onto a non-empty closed convex $\mathcal C\subseteq\mathbb R^N$. The \textcolor{black}{unit resource-allocation simplex} is $\Delta_N=\{x\in\mathbb R^N_{\ge 0}:\mathbf 1^\top x=1\}$, with diameter $R\le\sqrt 2$.

\textcolor{black}{For reference, Table~\ref{tab:notation} summarizes the principal notation used throughout the paper.}
\begin{table}[!t]
\begingroup
\color{black}
\caption{Principal notation. Vector inequalities are interpreted coordinatewise.}
\label{tab:notation}
\centering
\footnotesize
\setlength{\tabcolsep}{2.2pt}
\renewcommand{\arraystretch}{1.02}
\begin{tabular}{@{}p{0.36\columnwidth}p{0.58\columnwidth}@{}}
\hline
\textbf{Symbol(s)} & \textbf{Meaning}\\
\hline
$N,T,[q]$ & Number of sensors, horizon, and index set $[q]:=\{1,\ldots,q\}$.\\
$A_t,D,w,C_t$ & Pre-slot AoI, deadline and weight vectors, and pre-action service/refresh matrix.\\
$\Delta_N,\Pi_{\mathcal C},R$ & Unit resource-allocation simplex, Euclidean projection, and simplex diameter.\\
$x_t,z_t,\eta_t,Q_t$ & Executed safe action, pre-shield proposal, step size, and certificate queue.\\
$a_{t+1}^i(x),g_{t,i}(x)$ & Post-decision AoI and deadline-violation function for sensor $i$.\\
$\theta_{t,i}$ & Minimum refresh threshold required by sensor $i$ in slot $t$.\\
$\mathcal K_t,\mathcal K_{1:T}$ & Slot-wise safe set and its static intersection over the horizon.\\
$f_t,s_t,G$ & Convex slot loss, observed subgradient, and subgradient-norm bound.\\
$u,u_t,P_T$ & Static and dynamic safe comparators and comparator path length.\\
$\operatorname{Reg}_T,\operatorname{DReg}_T$ & Static and dynamic regret.\\
$\xi,\varepsilon_t,L_{t,i}$ & Safety margin, error radius, and constraint sensitivity.\\
$\mathcal U_t,\mathcal K_{t,\mathrm{rob}}$ & Channel uncertainty set and robust safe set.\\
\hline
\end{tabular}
\endgroup
\end{table}

\subsection{System Model and Task Definition}
\label{sec:task}

\textcolor{black}{A controller allocates a shared wireless uplink among $N$ transmitters over slots $t=1,\ldots,T$. At the beginning of slot $t$, it observes the current AoI vector $A_t$ and the safety side information: the effective refresh matrix $C_t$ in the nominal setting, or an uncertainty set $\mathcal U_t\ni C_t$ in the robust setting. This side information is obtained before the action from current pilots/channel-state measurements, packet timestamps and freshness indicators, or calibrated conservative bounds; no future channel or packet outcome is assumed. The controller then commits to $x_t\in\Delta_N$, where $x_{t,j}$ is the deterministic fraction of slot-$t$ airtime or bandwidth assigned to transmitter $j$ and $\mathbf 1^\top x_t=1$; $x_t$ is not a random draw of one transmitter. Figure~\ref{fig:system-timing} summarizes this model and the slot timing.} The matrix $C_t$ admits the optional decomposition
\begin{equation}
C_{t,ij}=B_{t,ij}\,p_{t,j}\,r_{t,j},
\label{eq:Ct-decomp}
\end{equation}
\textcolor{black}{where $B_{t,ij}=1$ means that a fresh update transmitted by $j$ also refreshes information item $i$ at the controller. Thus $B_t=I$ models independent sensors, whereas aggregation, correlated measurements, or overhearing can create off-diagonal ones. The factor $p_{t,j}\in[0,1]$ represents effective link service and $r_{t,j}\in[0,1]$ fresh-sample availability, so $(C_tx)_i\in[0,1]$ is the modeled refresh intensity of item $i$. Under deterministic fractional service, the resulting state is a fluid AoI state; when $p_{t,j}$ is a success probability, it is a conditional-mean AoI state.}

\begin{figure}[!t]
\centering
\begingroup
\color{black}
\begin{tikzpicture}[
  x=1cm,y=1cm,
  box/.style={draw=black,fill=black!3,rounded corners=1pt,
    align=center,inner sep=2pt,text=black},
  timebox/.style={box,minimum height=.84cm},
  flow/.style={-{Latex[length=1.35mm,width=.9mm]},
    line width=.35pt,draw=black},
  font=\scriptsize
]
\node[anchor=west,font=\scriptsize\bfseries,text=black] at (0,3.65)
  {(a) Physical meaning};
\node[box,text width=3.75cm] (share) at (2.02,3.05)
  {$x_{t,j}$: deterministic airtime/bandwidth share\\
   (not a random one-packet draw)};
\node[box,text width=3.75cm] (common) at (6.63,3.05)
  {$B_{t,ij}=1$: transmitter $j$'s fresh update\\
   also refreshes information item $i$; $B_t=I$ if independent};

\node[anchor=west,font=\scriptsize\bfseries,text=black] at (0,2.30)
  {(b) Slot-$t$ timing};
\node[timebox,text width=1.35cm] (s1) at (.72,1.28)
  {Reveal\\$A_t,C_t$\\or $\mathcal U_t$};
\node[timebox,text width=1.15cm] (s2) at (2.45,1.28)
  {Build\\$\mathcal K_t$};
\node[timebox,text width=1.25cm] (s3) at (4.18,1.28)
  {Shield\\$z_t\mapsto x_t$};
\node[timebox,text width=1.35cm] (s4) at (6.02,1.28)
  {Fluid service\\AoI update};
\node[timebox,text width=1.55cm] (s5) at (7.95,1.28)
  {Reveal $f_t,s_t$\\update $z_{t+1}$};
\draw[flow] (s1) -- (s2);
\draw[flow] (s2) -- (s3);
\draw[flow] (s3) -- (s4);
\draw[flow] (s4) -- (s5);
\draw[black,densely dashed,line width=.3pt] (5.05,.68) -- (5.05,1.92);
\node[font=\tiny,fill=white,inner sep=.7pt,text=black] at (5.05,1.93)
  {commit};
\node[font=\tiny\itshape,text=black] at (2.45,.56) {pre-action};
\node[font=\tiny\itshape,text=black] at (6.75,.56) {post-action};
\end{tikzpicture}
\caption{\textcolor{black}{System model and slot timing. The controller obtains $A_t$ and $C_t$ (or $\mathcal U_t\ni C_t$) before commitment, allocates deterministic resource shares $x_t$, advances the modeled AoI state, and receives first-order loss feedback only after execution.}}
\label{fig:system-timing}
\endgroup
\end{figure}
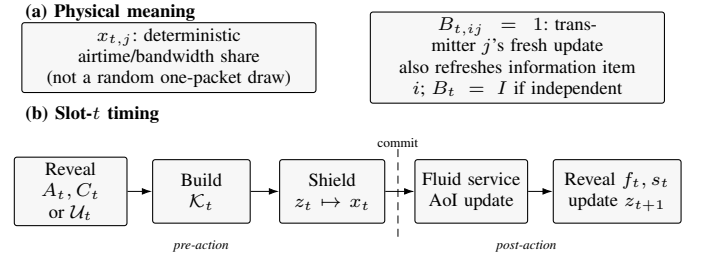

Executing $x\in\Delta_N$ produces the affine post-decision AoI
\begin{equation}
a_{t+1}^i(x)=1+\bigl(1-(C_tx)_i\bigr)A_t^i,
\qquad i\in[N],
\label{eq:aoi-dyn}
\end{equation}
and the system advances by $A_{t+1}^i=a_{t+1}^i(x_t)$.

\textcolor{black}{After commitment, the controller receives first-order OCO feedback $(f_t(x_t),s_t)$ with convex $f_t$ and $s_t\in\partial f_t(x_t)$. This is full-subgradient, not bandit, feedback; $s_t$ affects only the next proposal $z_{t+1}$.} The canonical instance is the weighted post-decision AoI,
\begin{align}
f_t^{\rm can}(x)
&=\sum_{i=1}^N w_i\,a_{t+1}^i(x)\notag\\
&=\sum_{i=1}^N w_i(A_t^i+1)-\bigl\langle C_t^\top(A_t\odot w),\,x\bigr\rangle,
\label{eq:can-loss}
\end{align}
which is affine in $x$, hence convex; any other convex slot-loss admitted by Assumption~\ref{ass:loss} is also allowed. The slot-$t$ deadline-violation function and refresh threshold are
\begin{align}
g_{t,i}(x) &= a_{t+1}^i(x)-D_i, \label{eq:g-def}\\
\theta_{t,i} &= \dfrac{A_t^i+1-D_i}{A_t^i}, \label{eq:theta-def}
\end{align}
and the current-slot safe set is
\begin{align}
\mathcal K_t
&=\{x\in\Delta_N:g_{t,i}(x)\le 0,\,\forall i\in[N]\}\notag\\
&=\{x\in\Delta_N:C_t x\ge\theta_t\},
\label{eq:Kt}
\end{align}
where the second equality follows from~\eqref{eq:g-def}--\eqref{eq:theta-def} together with $A_t^i\ge 1$. We also write $\mathcal K_{1:T}=\bigcap_{t=1}^T\mathcal K_t$ for the static safe set.

For any static safe comparator $u\in\mathcal K_{1:T}$, the trace-wise static regret is
\begin{equation}
\operatorname{Reg}_T(u)=\sum_{t=1}^T f_t(x_t)-\sum_{t=1}^T f_t(u),
\label{eq:regret}
\end{equation}
and for an arbitrary comparator sequence $u_t\in\mathcal K_t$ the dynamic regret and its path length are
\begin{align}
\operatorname{DReg}_T(u_{1:T})
&=\sum_{t=1}^T f_t(x_t)-\sum_{t=1}^T f_t(u_t),\label{eq:dreg}\\
P_T(u_{1:T})
&=\sum_{t=1}^{T-1}\|u_{t+1}-u_t\|_2.\label{eq:path}
\end{align}
We seek a causal algorithm that, \textcolor{black}{on every realized closed-loop trace satisfying Assumption~\ref{ass:viability}}, simultaneously delivers
\begin{equation}
\operatorname{Reg}_T(u)=O(\sqrt T)
\quad\text{and}\quad
\max_{t\in[T],\,i\in[N]}[A_t^i-D_i]_+=0.
\label{eq:goal}
\end{equation}

\subsection{Assumptions}
\label{sec:assumptions}

\begin{assumption}[Slotted system with causal side information]
\label{ass:causal}\label{ass:refresh}\label{ass:init}
At every slot $t\in[T]$ and before committing to $x_t$, the controller observes $A_t\in\mathbb R_+^N$ and $C_t\in[0,1]^{N\times N}$; the AoI evolves by~\eqref{eq:aoi-dyn}; and $1\le A_1^i\le D_i$ for every $i\in[N]$.
\end{assumption}

\begin{assumption}[Convex losses]
\label{ass:loss}
Each $f_t:\Delta_N\to\mathbb R$ is convex, and there exists $G<\infty$ with $\|s_t\|_2\le G$ for every $s_t\in\partial f_t(x_t)$ along the trajectory.
\end{assumption}

\begin{assumption}[Realized one-step viability]
\label{ass:viability}
For every $t\in[T]$, $\mathcal K_t\ne\varnothing$.
\end{assumption}

\begin{remark}[\textcolor{black}{Trajectory-dependent viability and emergency fallback}]
\label{rem:trajectory-viability}
\color{black}
Because $A_t$, $\theta_t$, and $\mathcal K_t$ depend on past actions, Assumption~\ref{ass:viability} concerns the algorithm's realized closed-loop trajectory, not $(C_t)_{t=1}^T$ alone; another policy may reach a different viable state under the same coefficients. If an encountered $\mathcal K_t$ is empty, no current action meets every modeled deadline from that state. The controller may then flag infeasibility and use, for example,
\[
x_t^{\rm emg}\in
\arg\min_{x\in\Delta_N}\max_{i\in[N]}[g_{t,i}(x)]_+,
\]
to minimize the largest unavoidable excess. The zero-violation and zero-queue results do not cover emergency operation unless viability is re-established.
\end{remark}

Assumption~\ref{ass:causal} is the natural causal information pattern of slotted wireless protocols; future $C_{t'}$ are not revealed, and a robust variant under uncertainty is treated in Section~\ref{sec:robust}. The bound $C_{t,ij}\in[0,1]$ is inherent in any probability or airtime interpretation, and the affine fluid dynamic~\eqref{eq:aoi-dyn} is the standard expected-refresh model; no stationarity, Markovianity, or mixing is imposed on $(C_t)$. The initial-safety part is necessary: $A_1^i>D_i$ already misses a deadline. Assumption~\ref{ass:loss} accommodates the canonical AoI loss, soft-max approximations of peak-AoI, and linear losses; when $A_t^i\le D_i$ along the trajectory (which our algorithm enforces), $f_t^{\rm can}$ is automatically Lipschitz with $G_{\rm can}\le\sqrt N\sum_i w_iD_i$. We do not require $\mathcal K_{1:T}\ne\varnothing$ except where the static-regret theorem explicitly states so.

\section{The OCO-PAoI-Hard Algorithm}
\label{sec:algorithm}

We now describe the proposed scheduler. Its top-level form is a strictly causal proposal-shield-update (PSU) loop:
\begin{equation}
z_t\ \xrightarrow{\ \Pi_{\mathcal K_t}\ }\ x_t\ \xrightarrow{\ x_t-\eta_t s_t,\ \Pi_{\Delta_N}\ }\ z_{t+1},
\label{eq:psu}
\end{equation}
in which an unconstrained OCO iterate $z_t\in\Delta_N$ is first hardened into the executed action $x_t$ by a single Euclidean projection onto the current-slot safe set $\mathcal K_t$ (the shield), and the gradient step is then anchored at $x_t$, not at $z_t$, to produce the next proposal. Compared with classical drift-plus-penalty schedulers, the loop separates three concerns that are usually conflated: \emph{state observation} (line~3), \emph{safety enforcement} (lines~4--5), and \emph{regret minimization} (lines~10--11). Algorithm~\ref{alg:main} states the resulting procedure, and the virtual queue $Q_t$ is retained only as an a-posteriori zero-violation certificate, not as a primary safety mechanism. We comment on each step below and discuss implementation aspects thereafter.

\begin{algorithm}[t]
\caption{OCO-PAoI-Hard: causal proposal-shield-update.}
\label{alg:main}
\begin{algorithmic}[1]
\REQUIRE deadlines $D\!\in\!\mathbb R_+^N$, weights $w\!\in\!\mathbb R_+^N$, step sizes $\{\eta_t\}_{t\ge 1}$, initial AoI $A_1$ with $A_1^i\!\in\![1,D_i]$, optional margin $\xi\!\in\!\mathbb R_+^N$, optional uncertainty sets $\{\mathcal U_t\}$
\ENSURE executed actions $(x_t)_{t=1}^T$, AoI trajectory $(A_t)_{t=1}^{T+1}$, queue trace $(Q_t)_{t=1}^{T+1}$
\STATE \textbf{initialize:} pick any $z_1\!\in\!\Delta_N$ (e.g., $z_1\!=\!\mathbf 1/N$); set $Q_1\!\gets\!\mathbf 0\in\mathbb R_+^N$
\FOR{$t=1,2,\dots,T$}
    \STATE \textbf{observe} current AoI vector $A_t$ and refresh matrix $C_t\in[0,1]^{N\times N}$ (or $\mathcal U_t\ni C_t$ in the robust case)
    \STATE \textbf{compute thresholds:} $\theta_{t,i}\!\gets\!\dfrac{A_t^i+1-D_i+\xi_i}{A_t^i}$ for $i\!\in\![N]$
    \STATE \textbf{build safe set:} nominal $\mathcal K_t^{\xi}\!\gets\!\{x\!\in\!\Delta_N:C_t x\!\ge\!\theta_t\}$, or robust $\mathcal K_{t,\rm rob}^\xi\!\gets\!\{x\!\in\!\Delta_N:C x\!\ge\!\theta_t,\,\forall C\!\in\!\mathcal U_t\}$ if $C_t$ is partially observable
    \STATE \textbf{safety shield (QP):} \\ \quad $x_t\gets\arg\min_{x\in\mathcal K_t^{\xi}}\tfrac{1}{2}\|x-z_t\|_2^2$ \hfill $\triangleright$ Euclidean projection
    \STATE \textbf{execute:} \textcolor{black}{allocate the deterministic resource shares $x_t$ and realize the modeled fluid service}
    \STATE \textbf{advance AoI:} $A_{t+1}^i\!\gets\!1+\bigl(1-(C_tx_t)_i\bigr)A_t^i$ for $i\!\in\![N]$
    \STATE \textbf{receive feedback:} convex loss $f_t:\Delta_N\!\to\!\mathbb R$ and subgradient $s_t\in\partial f_t(x_t)$ with $\|s_t\|_2\!\le\!G$
    \STATE \textbf{a-posteriori certificate:} $Q_{t+1,i}\!\gets\!\bigl[Q_{t,i}+g_{t,i}(x_t)\bigr]_+$ \hfill $\triangleright$ remains $0$ by Lemma~\ref{lem:queue}
    \STATE \textbf{gradient step (anchored at $x_t$):} $\widetilde z_{t+1}\!\gets\!x_t-\eta_t s_t$
    \STATE \textbf{simplex projection:} $z_{t+1}\!\gets\!\Pi_{\Delta_N}(\widetilde z_{t+1})$
\ENDFOR
\STATE \textbf{return} $(x_t,A_t,Q_t)_{t=1}^{T+1}$
\end{algorithmic}
\end{algorithm}

\subsection{Implementation Considerations}
\label{sec:implementation}

The dominant per-slot cost is the safe-set projection on line~5, a convex QP over the simplex with $N$ refresh inequalities, solvable in $O(N^3)$ by any active-set or interior-point method and amortizable by warm-starting from $z_t$. The simplex projection in line~10 takes $O(N\log N)$ via sort-and-threshold. For a known horizon $T$, the constant step size
\begin{equation}
\eta=\dfrac{R}{G\sqrt{T}},
\label{eq:eta-known}
\end{equation}
with $R\le\sqrt{2}$ the diameter of $\Delta_N$, optimizes the static-regret bound below, and the doubling-free schedule $\eta_t=R/(G\sqrt{t})$ achieves the same $O(\sqrt{T})$ order with a $3/2$ overhead when $T$ is unknown. For the canonical AoI loss, $G$ may be replaced by the data-only bound $G_{\rm can}\le\sqrt{N}\sum_i w_iD_i$. Anchoring the gradient step at the executed $x_t$ (line~10) rather than at $z_t$ is essential: it aligns the OCO recursion with the subgradient observed at $x_t$ and is what allows the projection inequality to telescope; using $z_t$ instead would couple the shield correction with the gradient direction across slots and break both the regret bound and causality. If only an uncertainty set $\mathcal U_t\ni C_t$ is available (e.g., a coordinatewise lower bound $\underline C_t$ from channel-state estimation), replacing $\mathcal K_t$ in line~4 with the robust safe set
\begin{equation}
\mathcal K_{t,\rm rob}=\{x\in\Delta_N:Cx\ge\theta_t,\,\forall C\in\mathcal U_t\}
\label{eq:Krob}
\end{equation}
preserves every guarantee for any realized $C_t\in\mathcal U_t$.

\subsection{Per-Step Reasoning}
\label{sec:perstep}

Lines~3--4 use \textcolor{black}{the pre-action observations $A_t$ and $C_t$ (or a conservative $\mathcal U_t\ni C_t$)} to form the affine threshold $\theta_t$ and the polyhedral safe set $\mathcal K_t$ via Lemma~\ref{lem:thr}; \textcolor{black}{no future packet outcome is used}. Line~5 hardens the proposal $z_t$, which encodes everything the algorithm has learned, into the closest feasible $x_t=\Pi_{\mathcal K_t}(z_t)$, and Lemma~\ref{lem:safety} guarantees $A_{t+1}^i\le D_i$ without further mechanism. Lines~6--7 \textcolor{black}{allocate deterministic resource shares and advance the modeled AoI state}. Lines~8--9 form the accounting layer: the convex loss $f_t$ and a subgradient $s_t$ become available only after the action commits, \textcolor{black}{which is first-order OCO rather than bandit feedback}; the queue update is recorded but, by Lemma~\ref{lem:queue}, never leaves zero, so it serves as a runtime certificate rather than a control variable. Lines~10--11 are the regret-minimization step: a one-step descent at $\eta_t s_t$ from the executed iterate $x_t$, followed by a single simplex projection. \textcolor{black}{Here ``learning'' means no-regret sequential adaptation from first-order feedback, not bandit exploration or channel-model learning.} Anchoring the descent at $x_t$ rather than at $z_t$ is what makes the squared-distance terms in Theorem~\ref{thm:main} telescope across slots; pulling the gradient back to $z_t$ would couple the shield correction with future gradients and break the $O(\sqrt T)$ rate.

\subsection{Shield Versus Long-Term Virtual Queue}
\label{sec:shield-vs-vq}

The role of the shield can be made precise by contrasting it with the standard long-term virtual-queue (LT-VQ) approach to constrained OCO. LT-VQ schedulers at every slot solve $\arg\min_x V f_t(x)+\sum_i Q_{t,i}g_{t,i}(x)$ over $\Delta_N$, where $V>0$ is a Lyapunov weight and $Q_t$ is a queue that absorbs accumulated constraint violation. With $V=\Theta(\sqrt T)$ this attains both $O(\sqrt T)$ regret and $O(\sqrt T)$ cumulative violation, but the queue is the only safety mechanism: at any slot in which the gradient direction conflicts with the deadline, $g_{t,i}(x_t)>0$ is permitted as long as the running sum $\sum_\tau[g_{\tau,i}(x_\tau)]_+$ remains sublinear. Proposition~\ref{prop:sublinear} formalizes the fundamental gap: even an $O(1)$ cumulative-violation budget is consistent with $g_{t_0,i}(x_{t_0})>0$ at some $t_0$, and a single such slot is enough to trip a $(1,1)$-firm peak-AoI constraint.

The PSU loop reverses the roles. The shield $\Pi_{\mathcal K_t}$ on line~5 enforces $g_{t,i}(x_t)\le 0$ in the slot itself, before $x_t$ is executed, and the queue update on line~9 is performed only as an audit trail. Lemmas~\ref{lem:safety}--\ref{lem:queue} together show that this design guarantees $Q_{t,i}=0$ \textcolor{black}{along the viable modeled trajectory}, while the gradient step retains the same $O(\sqrt T)$ rate as unconstrained OCD on $\Delta_N$. The trade-off between regret and safety that LT-VQ exposes through the parameter $V$ is therefore avoided entirely: the regret bound depends only on $R$, $G$, and $T$, not on any constraint-tightness parameter, and \textcolor{black}{the safety premise is that each safe set encountered along this trajectory is non-empty}.

%%% MAIN_RESULTS_PLACEHOLDER %%%
\section{Main Results}
\label{sec:main}

This section gives the complete theoretical analysis of Algorithm~\ref{alg:main}: the AoI-to-affine-safety reduction (\S\ref{sec:reduction}), the projection lemma (\S\ref{sec:projection}), the safety invariant and main theorem (\S\ref{sec:invariants}--\ref{sec:main-thm}), a path-length dynamic-regret bound (\S\ref{sec:dynamic}), a matching $\Omega(\sqrt T)$ lower bound (\S\ref{sec:lower-bound}), margin and approximate-projection robustness (\S\ref{sec:margin}--\ref{sec:approx}), the link to weakly-hard $(m,k)$-firm semantics (\S\ref{sec:weakly-hard}), the fractional-to-integral interface (\S\ref{sec:fractional}), a robust safe projection under partial channel observability (\S\ref{sec:robust}), and two competitive-ratio statements (\S\ref{sec:competitive}). All proofs are full and self-contained.

\subsection{AoI-to-Affine-Safety Reduction}
\label{sec:reduction}

\begin{lemma}[Deadline-to-threshold equivalence]
\label{lem:thr}
Suppose $A_t^i\ge 1$. Then for every $x\in\Delta_N$,
\begin{equation}
g_{t,i}(x)\le 0
\quad\Longleftrightarrow\quad
(C_tx)_i\ge\theta_{t,i},
\label{eq:thr-equiv}
\end{equation}
where $\theta_{t,i}=(A_t^i+1-D_i)/A_t^i$. Consequently, whenever $A_t^i\ge 1$ for every $i\in[N]$,
\begin{equation}
\mathcal{K}_t=\{x\in\Delta_N:C_tx\ge\theta_t\},
\label{eq:Kt-affine}
\end{equation}
read coordinatewise.
\end{lemma}

\begin{proof}
By~\eqref{eq:aoi-dyn}, $g_{t,i}(x)=A_t^i+1-D_i-A_t^i(C_tx)_i$, so $g_{t,i}(x)\le 0 \iff A_t^i+1-D_i\le A_t^i(C_tx)_i$. Since $A_t^i\ge 1>0$, dividing by $A_t^i$ preserves the inequality and gives $\theta_{t,i}\le(C_tx)_i$; each step is reversible. Conjoining over $i$ and intersecting with $\Delta_N$ yields~\eqref{eq:Kt-affine}.
\end{proof}

The hypothesis $A_t^i\ge 1$ is preserved along the trajectory of Algorithm~\ref{alg:main} by Lemma~\ref{lem:safety} below, whose induction relies only on Assumption~\ref{ass:init}, the dynamics~\eqref{eq:aoi-dyn}, and the projection step, so~\eqref{eq:Kt-affine} applies at every slot without circularity.

\begin{lemma}[Polyhedral structure of the safe set]
\label{lem:polyhedral}
Fix $(A_t,C_t)$ with $A_t^i\ge 1$ for every $i\in[N]$. The maps $x\mapsto a_{t+1}^i(x)$, $x\mapsto g_{t,i}(x)$, and $x\mapsto f_t^{\rm can}(x)$ are affine on $\mathbb R^N$. Consequently, $\mathcal K_t$ is a closed convex polyhedron, and is non-empty under Assumption~\ref{ass:viability}. Conversely, if $\mathcal K_t=\varnothing$, then every $x\in\Delta_N$ pushes some sensor's post-decision AoI above its deadline, so Assumption~\ref{ass:viability} characterizes exactly the sequences on which a zero-violation guarantee can hold.
\end{lemma}

\begin{proof}
By~\eqref{eq:aoi-dyn}, $a_{t+1}^i(x)=A_t^i+1-\langle c_{t,i},x\rangle$ with $c_{t,i}=A_t^i(C_{t,i1},\dots,C_{t,iN})^\top$, which is the sum of a constant and a linear functional, hence affine; $g_{t,i}=a_{t+1}^i-D_i$ remains affine, and $f_t^{\rm can}=\sum_i w_iA_t^i\bigl(1-(C_tx)_i\bigr)$ is a non-negative weighted sum of affine functions and is thus affine. By Lemma~\ref{lem:thr}, $\mathcal K_t=\Delta_N\cap\bigcap_{i=1}^N\{x:(C_tx)_i\ge\theta_{t,i}\}$. The simplex $\Delta_N$ is the intersection of the closed half-spaces $\{x_j\ge 0\}$ with the closed affine hyperplane $\{\mathbf 1^\top x=1\}$, and each $\{x:(C_tx)_i\ge\theta_{t,i}\}$ is a closed half-space, so $\mathcal K_t$ is a finite intersection of closed half-spaces and an affine hyperplane and is therefore a closed convex polyhedron. Non-emptiness follows from Assumption~\ref{ass:viability}; the converse is immediate from Lemma~\ref{lem:thr}.
\end{proof}

\begin{lemma}[Common observations only enlarge the safe set]
\label{lem:mono-obs}
Fix $A_t$ and $D$ and write $\mathcal K_t(C)=\{x\in\Delta_N:Cx\ge\theta_t\}$.
\textnormal{(i)} If $C'_t\ge C_t$ coordinatewise, then $\mathcal K_t(C_t)\subseteq\mathcal K_t(C'_t)$, and the inclusion is strict whenever some row of $C'_t-C_t$ is non-zero on the relative interior of $\Delta_N\cap\{x:(C_tx)_i=\theta_{t,i}\}$.
\textnormal{(ii)} Whenever it is non-empty, $\mathcal K_t$ is compact and equal to the convex hull of finitely many vertices, with diameter bounded by $\mathrm{diam}(\Delta_N)=\sqrt{2}$.
\end{lemma}

\begin{proof}
\textnormal{(i)} Take any $x\in\mathcal K_t(C_t)$. By Lemma~\ref{lem:thr}, $C_tx\ge\theta_t$. Since $C'_t-C_t\ge 0$ entry-wise and $x\ge 0$, $(C'_tx)_i-(C_tx)_i=\sum_{j}(C'_{t,ij}-C_{t,ij})x_j\ge 0$ for every $i$. Hence $C'_tx\ge C_tx\ge\theta_t$, and Lemma~\ref{lem:thr} gives $x\in\mathcal K_t(C'_t)$. The strictness claim follows because activating a row of $C'_t-C_t$ relaxes a binding face of $\mathcal K_t(C_t)$ into the interior of $\mathcal K_t(C'_t)$.
\textnormal{(ii)} Compactness is inherited from $\Delta_N$, and the Minkowski–Weyl theorem represents any non-empty bounded polyhedron as the convex hull of its finitely many extreme points. The diameter bound is inherited from $\Delta_N$.
\end{proof}

Lemma~\ref{lem:mono-obs} is the structural reason the robust safe set in \S\ref{sec:robust} retains feasibility under partial channel observability: replacing $C_t$ by any coordinatewise lower bound $\underline C_t\le C_t$ yields a tighter polyhedron $\mathcal K_t(\underline C_t)\subseteq\mathcal K_t(C_t)$, so projecting onto $\mathcal K_t(\underline C_t)$ is a fortiori safe under the realized channel. In conjunction with Lemma~\ref{lem:polyhedral}, it also implies that $\mathcal K_t$ admits a finite-vertex representation that an active-set or simplex-based projection routine can exploit, and that the projection in line~4 of Algorithm~\ref{alg:main} reduces to a quadratic program with at most $2N+1$ linear constraints.

\subsection{Projection Geometry}
\label{sec:projection}

\begin{lemma}[Existence, uniqueness, and Pythagorean inequality]
\label{lem:proj}
Let $\mathcal{C}\subset\mathbb{R}^N$ be non-empty, closed, and convex. For every $y\in\mathbb{R}^N$ the projection $p=\Pi_{\mathcal{C}}(y)=\arg\min_{x\in\mathcal{C}}\frac{1}{2}\|x-y\|_2^2$ exists and is unique, and for every $u\in\mathcal{C}$,
\begin{align}
\langle y-p,\,u-p\rangle &\le 0,\label{eq:vi}\\
\|p-u\|_2^2 &\le \|y-u\|_2^2-\|y-p\|_2^2\le\|y-u\|_2^2.\label{eq:pyth}
\end{align}
\end{lemma}

\begin{proof}
Existence and uniqueness follow by Weierstrass on the compact sublevel set $\mathcal{C}_0=\{x\in\mathcal{C}:\frac12\|x-y\|_2^2\le\frac12\|x_0-y\|_2^2\}$ with $x_0\in\mathcal C$, together with the parallelogram law: if $p\ne q$ both minimized $\frac12\|\cdot-y\|_2^2$, then $m=(p+q)/2\in\mathcal C$ would attain a strictly smaller value. For any $u\in\mathcal C$ and $\alpha\in[0,1]$, $p+\alpha(u-p)\in\mathcal C$, and minimality of $\alpha=0$ in $\frac12\|p+\alpha(u-p)-y\|_2^2$ gives~\eqref{eq:vi}. Decomposing $y-u=(y-p)+(p-u)$ and using~\eqref{eq:vi} yields~\eqref{eq:pyth}.
\end{proof}

\subsection{Safety Invariants and Zero-Violation Certificate}
\label{sec:invariants}

\begin{lemma}[Closed-form Lipschitz constant of the canonical loss]
\label{lem:lip}
If $A_t^i\le D_i$ for all $t,i$ along the trajectory, $f_t^{\rm can}$ is $G_{\rm can}$-Lipschitz on $\Delta_N$ with
\begin{equation}
G_{\rm can}=\sup_t\|C_t^\top(A_t\odot w)\|_2\le\sqrt{N}\sum_{i=1}^N w_iD_i,
\label{eq:Gcan}
\end{equation}
and, with $D_{\max}=\max_i D_i$, $W=\sum_i w_i$,
\begin{equation}
G_{\rm can}\le\min\{N D_{\max}\|w\|_2,\;\sqrt{N}\,D_{\max}W\}.
\label{eq:Gcan-alt}
\end{equation}
\end{lemma}

\begin{proof}
$\nabla f_t^{\rm can}=-C_t^\top(A_t\odot w)$ is constant in $x$, so by Cauchy--Schwarz $f_t^{\rm can}$ is Lipschitz with constant $\sup_t\|C_t^\top(A_t\odot w)\|_2$. Each entry $\sum_i C_{t,ij}A_t^iw_i$ is in $[0,\sum_i w_iD_i]$ since $C_{t,ij}\in[0,1]$ and $A_t^i\le D_i$, giving~\eqref{eq:Gcan}. For~\eqref{eq:Gcan-alt}, $\|C_t\|_F\le N$ and $\|A_t\odot w\|_2\le D_{\max}\|w\|_2$ yield the first bound, and $\sum_iw_iD_i\le D_{\max}W$ yields the second.
\end{proof}

\begin{lemma}[Stage-wise safety and well-posedness]
\label{lem:safety}\label{lem:wellposed}
Under Assumptions~\ref{ass:causal}--\ref{ass:viability}, Algorithm~\ref{alg:main} produces a unique $x_t=\Pi_{\mathcal{K}_t}(z_t)\in\mathcal{K}_t$ at every slot, and
\begin{equation}
1\le A_t^i\le D_i,\qquad t=1,\dots,T+1,\ i\in[N].
\label{eq:safe-state}
\end{equation}
\end{lemma}

\begin{proof}
We induct on $t$.

\emph{Base case} ($t=1$). Assumption~\ref{ass:init} gives $1\le A_1^i\le D_i$ for every $i\in[N]$, so~\eqref{eq:safe-state} holds at $t=1$.

\emph{Inductive step.} Suppose~\eqref{eq:safe-state} holds at slot $t$. Then $A_t^i\ge 1$ for all $i$, which activates the affine description $\mathcal K_t=\{x\in\Delta_N:C_tx\ge\theta_t\}$ in Lemma~\ref{lem:thr}. By Assumption~\ref{ass:viability}, $\mathcal K_t\ne\varnothing$, and by Lemma~\ref{lem:polyhedral} the set $\mathcal K_t$ is closed and convex. Lemma~\ref{lem:proj} therefore guarantees the existence of the unique projection $x_t=\Pi_{\mathcal K_t}(z_t)\in\mathcal K_t$, so the algorithm is well-posed. Since $x_t\in\mathcal K_t$, Lemma~\ref{lem:thr} gives $g_{t,i}(x_t)\le 0$ for every $i$, equivalently $A_{t+1}^i=1+(1-(C_tx_t)_i)A_t^i\le D_i$ by~\eqref{eq:aoi-dyn}. For the lower bound, $0\le(C_tx_t)_i\le 1$ because $C_t\in[0,1]^{N\times N}$ and $x_t\in\Delta_N$, and $A_t^i\ge 1$ by the inductive hypothesis, so $A_{t+1}^i\ge 1+0\cdot A_t^i=1$. Hence~\eqref{eq:safe-state} holds at $t+1$, completing the induction.
\end{proof}

\begin{lemma}[Virtual queue stays at zero]
\label{lem:queue}
With $Q_1=\mathbf{0}$, $Q_{t,i}=0$ and $\sum_{\tau=1}^t [g_{\tau,i}(x_\tau)]_+=0$ for all $t\in[T+1]$, $i\in[N]$.
\end{lemma}

\begin{proof}
We induct on $t$. The base case $Q_{1,i}=0$ holds by initialization. Suppose $Q_{t,i}=0$ for some $t\in[T]$. Lemma~\ref{lem:safety} gives $x_t\in\mathcal K_t$, so $g_{t,i}(x_t)\le 0$ by Lemma~\ref{lem:thr}. The line-9 update of Algorithm~\ref{alg:main} then yields $Q_{t+1,i}=[Q_{t,i}+g_{t,i}(x_t)]_+=[0+\text{non-positive}]_+=0$, completing the induction. The cumulative-violation identity follows from $[g_{\tau,i}(x_\tau)]_+=0$ at every $\tau$.
\end{proof}

\subsection{One-Step OCO Inequality and the Main Theorem}
\label{sec:main-thm}

\begin{lemma}[One-step OCO inequality after the safety shield]
\label{lem:onestep}
Under Assumptions~\ref{ass:causal}--\ref{ass:viability}, for any $u\in\mathcal{K}_t$,
\begin{equation}
f_t(x_t)-f_t(u)\le\dfrac{\|z_t-u\|_2^2-\|z_{t+1}-u\|_2^2}{2\eta_t}+\dfrac{\eta_t}{2}\|s_t\|_2^2.
\label{eq:onestep}
\end{equation}
\end{lemma}

\begin{proof}
Convexity gives $f_t(x_t)-f_t(u)\le\langle s_t,x_t-u\rangle$. Lemma~\ref{lem:proj} on $\mathcal{C}=\mathcal{K}_t$ gives $\|x_t-u\|_2^2\le\|z_t-u\|_2^2$, and on $\mathcal{C}=\Delta_N$ with $z_{t+1}=\Pi_{\Delta_N}(x_t-\eta_t s_t)$ gives $\|z_{t+1}-u\|_2^2\le\|x_t-\eta_t s_t-u\|_2^2$. Expanding,
\[
\|z_{t+1}-u\|_2^2\le\|z_t-u\|_2^2-2\eta_t\langle s_t,x_t-u\rangle+\eta_t^2\|s_t\|_2^2.
\]
Solving for the inner product, dividing by $2\eta_t$, and combining yields~\eqref{eq:onestep}.
\end{proof}

\begin{theorem}[OCO-PAoI-Hard: hard safety, zero violation, $O(\sqrt{T})$ regret]
\label{thm:main}
Under Assumptions~\ref{ass:causal}--\ref{ass:viability} and $\|s_t\|_2\le G$, with non-increasing $\eta_1\ge\cdots\ge\eta_T>0$, for every adversarial sequence $(C_t,f_t)_{t=1}^T$ Algorithm~\ref{alg:main} pathwise satisfies $A_t^i\le D_i$ for all $t,i$ and $Q_{t,i}=0$, $\sum_{\tau=1}^t[g_{\tau,i}(x_\tau)]_+=0$. If $\mathcal{K}_{1:T}=\bigcap_{t=1}^T\mathcal{K}_t\ne\varnothing$, then for every $u\in\mathcal{K}_{1:T}$,
\begin{equation}
\operatorname{Reg}_T(u)\le\dfrac{R^2}{2\eta_T}+\dfrac{1}{2}\sum_{t=1}^T\eta_t\|s_t\|_2^2.
\label{eq:reg-general}
\end{equation}
With known horizon $T$ and $\eta=R/(G\sqrt{T})$,
\begin{equation}
\operatorname{Reg}_T(u)\le RG\sqrt{T}\le\sqrt{2}\,G\sqrt{T}.
\label{eq:reg-known}
\end{equation}
With unknown horizon and $\eta_t=R/(G\sqrt{t})$,
\begin{equation}
\operatorname{Reg}_T(u)\le\tfrac{3}{2}\,RG\sqrt{T}.
\label{eq:reg-unknown}
\end{equation}
\end{theorem}

\begin{proof}
The hard-safety statement is Lemma~\ref{lem:safety} and the zero-violation certificate is Lemma~\ref{lem:queue}. Fix any $u\in\mathcal{K}_{1:T}$, so $u\in\mathcal{K}_t$ for every $t$ and Lemma~\ref{lem:onestep} gives, with $a_t:=\|z_t-u\|_2^2\in[0,R^2]$,
\begin{equation}
f_t(x_t)-f_t(u)\le\dfrac{a_t-a_{t+1}}{2\eta_t}+\dfrac{\eta_t}{2}\|s_t\|_2^2.
\label{eq:reg-1}
\end{equation}
Summing over $t\in[T]$,
\begin{equation}
\operatorname{Reg}_T(u)\le\sum_{t=1}^T\dfrac{a_t-a_{t+1}}{2\eta_t}+\dfrac{1}{2}\sum_{t=1}^T\eta_t\|s_t\|_2^2.
\label{eq:reg-2}
\end{equation}
Apply Abel summation to the first term,
\begin{equation}
\sum_{t=1}^T\dfrac{a_t-a_{t+1}}{2\eta_t}=\dfrac{a_1}{2\eta_1}-\dfrac{a_{T+1}}{2\eta_T}+\sum_{t=2}^T a_t\Big(\dfrac{1}{2\eta_t}-\dfrac{1}{2\eta_{t-1}}\Big),
\label{eq:reg-3}
\end{equation}
where the brackets are nonnegative since $\eta_t$ is non-increasing. Using $a_t\le R^2$ and $-a_{T+1}/(2\eta_T)\le 0$,
\begin{equation}
\sum_{t=1}^T\dfrac{a_t-a_{t+1}}{2\eta_t}\le\dfrac{R^2}{2\eta_T},
\label{eq:reg-4}
\end{equation}
which combined with~\eqref{eq:reg-2} yields~\eqref{eq:reg-general}. With $\eta=R/(G\sqrt T)$, the right-hand side of~\eqref{eq:reg-general} becomes $RG\sqrt T$, giving~\eqref{eq:reg-known}; with $\eta_t=R/(G\sqrt t)$ the second term is $\frac{RG}{2}\sum_{t=1}^T t^{-1/2}\le RG\sqrt T$ via $\sum t^{-1/2}\le 2\sqrt T$, totaling~\eqref{eq:reg-unknown}.
\end{proof}

\textcolor{black}{Theorem~\ref{thm:main} is trajectory-wise for deterministic fluid service and slotwise in conditional expected AoI when $C_t$ contains success probabilities; the latter is not a Bernoulli packet sample-path claim.} The zero-violation certificate $Q_{t,i}=0$ needs no knowledge of $\mathcal K_{1:T}$, and the regret bound matches unconstrained OCO up to $R/\eta$. The known-horizon constant $\sqrt{2}\,G\sqrt T$ is independent of $N$; the anytime schedule pays only a $3/2$ factor and is used in experiments.

\subsection{Dynamic Regret without a Static Comparator}
\label{sec:dynamic}

When $\mathcal{K}_{1:T}=\varnothing$, no fixed safe comparator exists. Algorithm~\ref{alg:main} still attains a path-length dynamic-regret bound for any safe comparator sequence.

\begin{theorem}[Path-length dynamic regret]
\label{thm:dyn}
Under Assumptions~\ref{ass:causal}--\ref{ass:viability} and $\|s_t\|_2\le G$, with constant step size $\eta>0$, for any $u_t\in\mathcal{K}_t$,
\begin{equation}
\operatorname{DReg}_T(u_{1:T})\le\dfrac{R^2+2R P_T(u_{1:T})}{2\eta}+\dfrac{\eta G^2 T}{2}.
\label{eq:dreg-1}
\end{equation}
With $\eta=R/(G\sqrt{T})$,
\begin{equation}
\operatorname{DReg}_T(u_{1:T})\le RG\sqrt{T}+G\sqrt{T}\,P_T(u_{1:T}).
\label{eq:dreg-2}
\end{equation}
\end{theorem}

\begin{proof}
For each $t$, $u_t\in\mathcal{K}_t$, so Lemma~\ref{lem:onestep} gives
\[
f_t(x_t)-f_t(u_t)\le\dfrac{\|z_t-u_t\|_2^2-\|z_{t+1}-u_t\|_2^2}{2\eta}+\dfrac{\eta}{2}\|s_t\|_2^2.
\]
Summing over $t\in[T]$,
\begin{equation}
\operatorname{DReg}_T(u_{1:T})\le\dfrac{S}{2\eta}+\dfrac{\eta}{2}\sum_{t=1}^T\|s_t\|_2^2,
\label{eq:dreg-3}
\end{equation}
with $S=\sum_{t=1}^T\bigl(\|z_t-u_t\|_2^2-\|z_{t+1}-u_t\|_2^2\bigr)$. Inserting telescoping terms in $\|z_{t+1}-u_{t+1}\|_2^2$,
\begin{align}
S&=\|z_1-u_1\|_2^2-\|z_{T+1}-u_T\|_2^2\notag\\
&\quad+\sum_{t=1}^{T-1}\bigl(\|z_{t+1}-u_{t+1}\|_2^2-\|z_{t+1}-u_t\|_2^2\bigr).\label{eq:dreg-4}
\end{align}
The first two terms contribute at most $R^2$ since $z_t,u_t\in\Delta_N$. For each summand in the third term, the identity $\|a-b\|_2^2-\|a-c\|_2^2=\langle c-b,2a-b-c\rangle$ with $a=z_{t+1}$, $b=u_{t+1}$, $c=u_t$, and Cauchy--Schwarz, give
\begin{align}
&\|z_{t+1}-u_{t+1}\|_2^2-\|z_{t+1}-u_t\|_2^2\notag\\
&\quad\le\|u_t-u_{t+1}\|_2\cdot\|2z_{t+1}-u_{t+1}-u_t\|_2,
\label{eq:dreg-5}
\end{align}
and the triangle inequality bounds the second factor by $2R$. Therefore
\begin{equation}
S\le R^2+2R\sum_{t=1}^{T-1}\|u_{t+1}-u_t\|_2=R^2+2R P_T(u_{1:T}).
\label{eq:dreg-6}
\end{equation}
Combining~\eqref{eq:dreg-6} with~\eqref{eq:dreg-3} and $\sum_t\|s_t\|_2^2\le G^2T$ yields~\eqref{eq:dreg-1}, and substituting $\eta=R/(G\sqrt T)$ gives~\eqref{eq:dreg-2}.
\end{proof}

The path-length term $G\sqrt{T}\,P_T$ in~\eqref{eq:dreg-2} is unavoidable in adversarial OCO and degenerates to the static bound when $u_{1:T}$ is constant. Theorem~\ref{thm:dyn} therefore says the algorithm tracks any drifting safe target with regret that scales linearly in how much the target moves, while the safety conclusions of Theorem~\ref{thm:main} are unaffected: even when $\mathcal K_{1:T}=\varnothing$ and a static comparator does not exist, every realized $x_t$ stays in $\mathcal K_t$ and meets every per-slot peak-AoI deadline.

\subsection{Matching Minimax Lower Bound}
\label{sec:lower-bound}

\begin{theorem}[$\Omega(\sqrt{T})$ minimax lower bound]
\label{thm:lower}
There exist instances of OCO-PAoI-Hard, with $N=2$, $D_i\ge T+2$ (so $\mathcal{K}_t=\Delta_N$ for all $t$), and convex losses satisfying Assumption~\ref{ass:loss}, on which any (possibly randomized) online algorithm incurs expected static regret $\Omega(\sqrt{T})$.
\end{theorem}

\begin{proof}
Set $N=2$, $D_i\ge T+2$, and $A_1^i=1$. Then $\theta_{t,i}\le 0$ and $\mathcal{K}_t=\Delta_2$ for every $t\in[T]$, so the safety constraint is inactive and the problem reduces to standard OCO over $\Delta_2$ with bounded subgradients. Let $\sigma_t\in\{-1,+1\}$ be i.i.d.\ Rademacher random variables, and consider the convex losses
\begin{equation}
f_t(x)=\sigma_t x_1.
\label{eq:lower-1}
\end{equation}
Each $\nabla f_t=\sigma_t e_1$ has norm $1$, so Assumption~\ref{ass:loss} holds with $G=1$. For any deterministic algorithm, $x_t$ is a function of $\sigma_1,\dots,\sigma_{t-1}$ and is therefore independent of $\sigma_t$, so
\begin{equation}
\mathbb{E}[\sigma_t x_{t,1}]=\mathbb{E}[x_{t,1}]\,\mathbb{E}[\sigma_t]=0,
\quad\mathbb{E}\Big[\sum_{t=1}^T f_t(x_t)\Big]=0.
\label{eq:lower-2}
\end{equation}
Set $S_T=\sum_{t=1}^T\sigma_t$. The static comparator's loss is $\sum_{t=1}^T f_t(u)=u_1 S_T$, and minimizing over $u\in\Delta_2$ gives
\begin{equation}
\min_{u\in\Delta_2}\sum_{t=1}^T f_t(u)=\min\{0,S_T\}.
\label{eq:lower-3}
\end{equation}
Combining~\eqref{eq:lower-2}--\eqref{eq:lower-3} and using the symmetry of $S_T$,
\begin{equation}
\mathbb{E}[\operatorname{Reg}_T]=-\mathbb{E}[\min\{0,S_T\}]=\tfrac{1}{2}\mathbb{E}|S_T|.
\label{eq:lower-4}
\end{equation}
Independence and $\mathbb E[\sigma_t]=0$, $\mathbb E[\sigma_t^2]=1$ give
\begin{equation}
\mathbb{E}[S_T^2]=\sum_{t=1}^T\mathbb{E}[\sigma_t^2]+\sum_{t\ne s}\mathbb{E}[\sigma_t]\mathbb{E}[\sigma_s]=T.
\label{eq:lower-mom2}
\end{equation}
For the fourth moment, $\mathbb E[\sigma_{t_1}\sigma_{t_2}\sigma_{t_3}\sigma_{t_4}]\ne 0$ only if every index appears an even number of times: all four indices equal contribute $T$ terms, and two distinct pairs contribute $3T(T-1)$ ordered quadruples each with expectation $1$, hence
\begin{equation}
\mathbb{E}[S_T^4]=T+3T(T-1)=3T^2-2T\le 3T^2.
\label{eq:lower-mom4}
\end{equation}
Apply the Paley--Zygmund inequality to $Y=S_T^2\ge 0$ with $\vartheta=1/2$,
\begin{equation}
\mathbb{P}\Big(Y\ge\tfrac{1}{2}\mathbb{E}Y\Big)\ge(1-\vartheta)^2\dfrac{(\mathbb{E}Y)^2}{\mathbb{E}[Y^2]}\ge\dfrac{1}{4}\cdot\dfrac{T^2}{3T^2}=\dfrac{1}{12},
\label{eq:lower-5}
\end{equation}
equivalent to $\mathbb{P}(|S_T|\ge\sqrt{T/2})\ge 1/12$. Hence $\mathbb E|S_T|\ge\sqrt{T/2}/12=\sqrt T/(12\sqrt 2)$, and~\eqref{eq:lower-4} gives $\mathbb E[\operatorname{Reg}_T]\ge\sqrt T/(24\sqrt 2)$. Yao's principle extends the bound to randomized algorithms.
\end{proof}

The lower bound holds in a regime where Assumption~\ref{ass:viability} is trivially satisfied, so the difficulty is purely from online learning, not from constraints; adding the affine safe-set machinery costs nothing in regret order. Together with the upper bound $\frac{3}{2}RG\sqrt{T}$ of Theorem~\ref{thm:main}, this matches up to constants and shows that no online algorithm with the same observation set can achieve a better-than-$\sqrt{T}$ rate, even with a hard real-time deadline structure layered on top of the OCO problem.

\subsection{Margin and Approximate Projection}
\label{sec:margin}\label{sec:approx}

For $\xi\in\mathbb{R}_+^N$, define
\begin{align}
\theta_{t,i}^\xi &=\dfrac{A_t^i+1-D_i+\xi_i}{A_t^i},\label{eq:thr-margin}\\
\mathcal{K}_t^\xi &=\{x\in\Delta_N:g_{t,i}(x)\le-\xi_i,\,\forall i\}\notag\\
&=\{x\in\Delta_N:C_tx\ge\theta_t^\xi\}.\label{eq:Kmargin}
\end{align}

\begin{theorem}[Margin-safe variant]
\label{thm:margin}
Suppose $\xi\in\mathbb{R}_+^N$ satisfies
\begin{equation}
0\le\xi_i\le D_i-1,\quad 1\le A_1^i\le D_i-\xi_i,\quad i\in[N],
\label{eq:margin-cond}
\end{equation}
and $\mathcal{K}_t^\xi\ne\varnothing$ for every $t$. If Algorithm~\ref{alg:main} replaces $\mathcal{K}_t$ by $\mathcal{K}_t^\xi$, then all conclusions of Theorem~\ref{thm:main} continue to hold and the safety bound strengthens to
\begin{equation}
1\le A_t^i\le D_i-\xi_i,\qquad t=1,\dots,T+1.
\label{eq:margin-safety}
\end{equation}
The regret bound applies for any $u\in\bigcap_{t=1}^T\mathcal{K}_t^\xi$.
\end{theorem}

\begin{proof}
The proof of Lemma~\ref{lem:safety} carries over verbatim with $\mathcal K_t$ replaced by $\mathcal K_t^\xi$, giving $g_{t,i}(x_t)\le-\xi_i$ and hence $A_{t+1}^i\le D_i-\xi_i$. The interval $[1,D_i-\xi_i]$ is non-empty since $\xi_i\le D_i-1$. The proof of Lemma~\ref{lem:onestep}, and hence of~\eqref{eq:reg-general}, is unchanged when $\mathcal C=\mathcal K_t^\xi$ in Lemma~\ref{lem:proj}.
\end{proof}

The margin variant inflates the affine threshold by $\xi_i/A_t^i$, shrinking $\mathcal K_t^\xi\subseteq\mathcal K_t$ so that every executed action sits at distance at least $\xi_i$ from each deadline face. The same regret bound applies because the polyhedral projection geometry depends only on closed convexity, not on the specific affine offsets.

\begin{theorem}[Margin shields against execution perturbations]
\label{thm:margin-eps}
Suppose Algorithm~\ref{alg:main} computes $x_t\in\mathcal{K}_t^\xi$ but executes $\widehat{x}_t\in\Delta_N$ with $\|\widehat{x}_t-x_t\|_2\le\varepsilon_t$. Define
\begin{equation}
L_{t,i}=\|\nabla g_{t,i}\|_2=A_t^i\|C_{t,i:}\|_2.
\label{eq:Lti}
\end{equation}
If $\xi_i\ge L_{t,i}\varepsilon_t$ for every $i\in[N]$ and $t\in[T]$, then $g_{t,i}(\widehat{x}_t)\le 0$, and in particular $A_{t+1}^i\le D_i$ along the executed trajectory.
\end{theorem}

\begin{proof}
$g_{t,i}$ is affine with $\nabla g_{t,i}=-A_t^i C_{t,i:}^\top$, so $\|\nabla g_{t,i}\|_2=L_{t,i}$. By Cauchy--Schwarz, $g_{t,i}(\widehat x_t)\le g_{t,i}(x_t)+L_{t,i}\|\widehat x_t-x_t\|_2\le-\xi_i+L_{t,i}\varepsilon_t\le 0$.
\end{proof}

The condition $\xi\ge L\varepsilon$ converts an analytic margin into resilience against arbitrary execution noise of magnitude $\varepsilon_t$, including quantization, integral rounding, and physical actuation error. Because $L_{t,i}\le\sqrt{N}D_i$ for the canonical loss, a margin scaling as $\sqrt{N}D_{\max}\varepsilon$ is enough; the cost is only an $O(\xi)$ shrinkage of the safe set, not a change in regret rate.

\begin{theorem}[Feasible approximate projection]
\label{thm:approx}
If at every slot the oracle returns a feasible $\widetilde{x}_t\in\mathcal{K}_t$ with $\|\widetilde{x}_t-\Pi_{\mathcal{K}_t}(z_t)\|_2\le\varepsilon_t$, and the proposal updates by $z_{t+1}=\Pi_{\Delta_N}(\widetilde{x}_t-\eta s_t)$, then hard safety still holds and for any $u\in\mathcal{K}_{1:T}$,
\begin{equation}
\operatorname{Reg}_T(u)\le\dfrac{R^2}{2\eta}+\dfrac{\eta G^2 T}{2}+\dfrac{1}{2\eta}\sum_{t=1}^T(2R\varepsilon_t+\varepsilon_t^2).
\label{eq:approx-reg}
\end{equation}
\end{theorem}

\begin{proof}
Since $\widetilde x_t\in\mathcal K_t$, Lemma~\ref{lem:thr} gives $g_{t,i}(\widetilde x_t)\le 0$, so $A_{t+1}^i\le D_i$ by~\eqref{eq:aoi-dyn}; the lower bound $A_{t+1}^i\ge 1$ and the queue-zero certificate follow verbatim from Lemmas~\ref{lem:safety} and~\ref{lem:queue} applied with $\widetilde x_t$ in place of $x_t$.

For the regret, let $p_t=\Pi_{\mathcal{K}_t}(z_t)$. For $u\in\mathcal{K}_t$, Lemma~\ref{lem:proj} gives $\|p_t-u\|_2\le\|z_t-u\|_2$. By the triangle inequality and $\|\widetilde{x}_t-p_t\|_2\le\varepsilon_t$,
\begin{equation}
\|\widetilde{x}_t-u\|_2\le\|p_t-u\|_2+\varepsilon_t\le\|z_t-u\|_2+\varepsilon_t.
\label{eq:approx-1}
\end{equation}
Squaring and using $\|z_t-u\|_2\le R$,
\begin{equation}
\|\widetilde{x}_t-u\|_2^2\le\|z_t-u\|_2^2+2R\varepsilon_t+\varepsilon_t^2.
\label{eq:approx-2}
\end{equation}
The proposal-update projection together with $u\in\mathcal{K}_t\subseteq\Delta_N$ and Lemma~\ref{lem:proj} on $\mathcal{C}=\Delta_N$ give $\|z_{t+1}-u\|_2^2\le\|\widetilde{x}_t-\eta s_t-u\|_2^2$. Expanding and substituting~\eqref{eq:approx-2},
\begin{align}
\langle s_t,\widetilde{x}_t-u\rangle
&\le\dfrac{\|z_t-u\|_2^2-\|z_{t+1}-u\|_2^2}{2\eta}+\dfrac{\eta}{2}\|s_t\|_2^2\notag\\
&\quad+\dfrac{2R\varepsilon_t+\varepsilon_t^2}{2\eta}.\label{eq:approx-3}
\end{align}
Convexity at $\widetilde x_t$ gives $f_t(\widetilde x_t)-f_t(u)\le\langle s_t,\widetilde x_t-u\rangle$. Summing~\eqref{eq:approx-3} over $t$, the first term telescopes to at most $R^2/(2\eta)$, the second to $\eta G^2T/2$, and the third to $\frac{1}{2\eta}\sum_t(2R\varepsilon_t+\varepsilon_t^2)$. Combining yields~\eqref{eq:approx-reg}.
\end{proof}

A margin $\xi\ge L\varepsilon$ buys \textcolor{black}{modeled-state safety} against execution noise, and a feasible projection oracle accurate to $\varepsilon_t=O(1/\sqrt T)$ preserves the $O(\sqrt T)$ regret order. Both together let the system designer trade a small amount of slack for resilience to numerical and physical noise without changing the rate.

\subsection{Weakly-Hard Real-Time Implications}
\label{sec:weakly-hard}

\begin{definition}[AoI-domain $(m,k)$-firm semantics]
\label{def:mk}
With miss indicator $v_t^i=\mathbf{1}\{A_t^i>D_i\}$, sensor $i$ is $(m_i,k_i)$-firm if $\sum_{\tau=t}^{t+k_i-1}(1-v_\tau^i)\ge m_i$ for $t=1,\dots,T-k_i+1$.
\end{definition}

\begin{corollary}[PAoI safety implies all $(m,k)$-firm requirements]
\label{cor:mk}\label{cor:11firm}
Under Theorem~\ref{thm:main}, every sensor $i$ is $(m_i,k_i)$-firm for every $k_i\in\mathbb N$ and $m_i\le k_i$. Conversely, $(1,1)$-firm at every slot is equivalent to $A_t^i\le D_i$ everywhere.
\end{corollary}

\begin{proof}
By Theorem~\ref{thm:main}, $A_t^i\le D_i$ at every slot $t$ along the trajectory of Algorithm~\ref{alg:main}, hence $v_t^i=\mathbf 1\{A_t^i>D_i\}=0$ for all $t$. For any window $[t,t+k_i-1]$, $\sum_{\tau=t}^{t+k_i-1}(1-v_\tau^i)=k_i\ge m_i$ whenever $m_i\le k_i$, so sensor $i$ is $(m_i,k_i)$-firm by Definition~\ref{def:mk}. The converse direction follows by definition: $(1,1)$-firm requires $1-v_t^i\ge 1$, i.e., $v_t^i=0$, equivalently $A_t^i\le D_i$ at every slot, which is exactly peak-AoI safety.
\end{proof}

Corollary~\ref{cor:mk} shows that hard peak-AoI safety is the strongest weakly-hard guarantee on the AoI-firm hierarchy: it implies $(m,k)$-firm for every admissible pair, and the strictest case $(1,1)$-firm coincides with peak-AoI safety. Algorithms that achieve only sublinear cumulative violation cannot certify $(1,1)$-firm, which is precisely the failure mode flagged by Proposition~\ref{prop:sublinear}.

\subsection{Fractional-to-Integral Safety Interface}
\label{sec:fractional}

\begin{theorem}[Necessary and sufficient condition for safe single-packet action]
\label{thm:integral}
Define $\mathcal{J}_t=\{j\in[N]:e_j\in\mathcal{K}_t\}$. A safe integral action exists at slot $t$ iff $\exists j$ with $C_{t,ij}\ge\theta_{t,i}$ for all $i$, equivalently $\mathcal J_t\ne\varnothing$. A randomized integral implementation has sample-path PAoI safety iff its support is contained in $\mathcal{J}_t$.
\end{theorem}

\begin{proof}
For any $j\in[N]$, $C_te_j$ is the $j$-th column of $C_t$, so its $i$-th coordinate is $C_{t,ij}$. By Lemma~\ref{lem:thr}, $e_j\in\mathcal K_t$ if and only if $C_{t,ij}\ge\theta_{t,i}$ for every $i\in[N]$, which establishes the first equivalence.

For the randomized statement, let $I_t\sim\nu$ be a random integral action with distribution $\nu$ on $[N]$. If $\mathrm{supp}(\nu)\subseteq\mathcal J_t$, then every realization $I_t=j$ satisfies $e_j\in\mathcal K_t$ and hence $g_{t,i}(e_j)\le 0$, so the executed AoI satisfies $A_{t+1}^i\le D_i$ on every sample path. Conversely, suppose $\mathrm{supp}(\nu)\not\subseteq\mathcal J_t$ and pick $j_0\in\mathrm{supp}(\nu)\setminus\mathcal J_t$. Then $\nu(\{j_0\})>0$, $e_{j_0}\notin\mathcal K_t$, and there exists $i_0$ with $g_{t,i_0}(e_{j_0})>0$, so the event $I_t=j_0$ has positive probability and produces $A_{t+1}^{i_0}>D_{i_0}$, breaking sample-path safety.
\end{proof}

\textcolor{black}{If $C_t$ contains expected success coefficients instead, $e_j\in\mathcal K_t$ certifies only conditional expected AoI. Let $\mathcal O_t$ be the set of possible realized outcome matrices. Packet-level sample-path safety requires the stronger outcome-wise set
\[
\mathcal J_t^{\rm out}
=\{j\in[N]:Ce_j\ge\theta_t,\ \forall C\in\mathcal O_t\}
\]
to be non-empty, and any randomized implementation must be supported on $\mathcal J_t^{\rm out}$. Deterministic service, worst-case service envelopes, or sufficient redundancy can make this condition hold; an expectation-based safe set alone cannot.}

A fractional safe action does not in general imply randomized integral safety. As a one-slot two-sensor counterexample, take $N=2$, $A_t^1=A_t^2=2$, $D_1=D_2=2$, $C_t=I_2$: then $\theta_t=(1/2,1/2)$ and the fractional action $x=(1/2,1/2)\in\mathcal K_t$, since $C_tx=(1/2,1/2)=\theta_t$, but the only two integral actions $C_te_1=(1,0)$ and $C_te_2=(0,1)$ each leave one sensor unrefreshed and yield $A_{t+1}^{i}=3>D_i$ for that sensor. Hence $\mathcal J_t=\varnothing$ even though $\mathcal K_t\ne\varnothing$, so any randomized rounding produces a deadline miss with probability one. \textcolor{black}{Thus support-restricted rounding is safe only if $\mathcal J_t$ (or $\mathcal J_t^{\rm out}$ under outcome uncertainty) is non-empty; increasing the safety margin cannot create integral feasibility.}

\subsection{Robust Safe Projection Under Partial Channel Observability}
\label{sec:robust}

\begin{definition}[Robust safe set]
\label{def:robust}
If only $\mathcal{U}_t\subseteq[0,1]^{N\times N}$ with $C_t\in\mathcal{U}_t$ is known at slot $t$,
\begin{equation}
\mathcal{K}_{t,\rm rob}=\{x\in\Delta_N:Cx\ge\theta_t,\,\forall C\in\mathcal{U}_t\}.
\label{eq:Krob-def}
\end{equation}
\end{definition}

\begin{theorem}[Robust safety and regret]
\label{thm:robust}
If $\mathcal{K}_{t,\rm rob}\ne\varnothing$ for every $t$ and Algorithm~\ref{alg:main} uses $\mathcal{K}_{t,\rm rob}$ in place of $\mathcal{K}_t$, then for every realized sequence $C_t\in\mathcal{U}_t$ all conclusions of Theorem~\ref{thm:main} continue to hold; the static-regret bound applies for any $u\in\bigcap_t\mathcal{K}_{t,\rm rob}$.
\end{theorem}

\begin{proof}
$\mathcal K_{t,\rm rob}$ is non-empty closed convex (intersection of $\Delta_N$ with closed half-spaces), so Lemma~\ref{lem:proj} applies and $x_t\in\mathcal K_{t,\rm rob}$ implies $C_tx_t\ge\theta_t$ since $C_t\in\mathcal U_t$. Lemma~\ref{lem:thr} gives $g_{t,i}(x_t)\le 0$, so the safety induction of Lemma~\ref{lem:safety} carries through. The regret proof is identical to Theorem~\ref{thm:main} with $\mathcal C=\mathcal K_{t,\rm rob}$ in Lemma~\ref{lem:proj}.
\end{proof}

\begin{corollary}[Rectangular lower-bound uncertainty]
\label{cor:rect}
If $\mathcal{U}_t=\{C\in[0,1]^{N\times N}:C\ge\underline{C}_t\}$, then $\mathcal{K}_{t,\rm rob}=\{x\in\Delta_N:\underline{C}_t x\ge\theta_t\}$.
\end{corollary}

\begin{proof}
$\underline{C}_tx\ge\theta_t$ with $C\ge\underline C_t$ and $x\ge 0$ gives $Cx\ge\theta_t$, so $x\in\mathcal K_{t,\rm rob}$. The converse follows since $\underline C_t\in\mathcal U_t$.
\end{proof}

Theorem~\ref{thm:robust} closes the gap between idealized analysis and practical deployment, where the controller often only has access to channel state estimates. Lemma~\ref{lem:mono-obs} guarantees that any coordinatewise lower envelope $\underline C_t\le C_t$ produces a tighter, safer polyhedron, and Corollary~\ref{cor:rect} singles out the rectangular case in which the robust projection collapses to a standard QP with $\underline C_t$ in place of $C_t$, so the robust algorithm has the same per-slot complexity as the nominal one.

\subsection{Closed Competitive-Ratio Statements}
\label{sec:competitive}

In the canonical-loss regime, the cumulative algorithmic cost equals the weighted AoI:
\begin{equation}
J_T^{\rm alg}=\sum_{t=1}^T\sum_{i=1}^N w_iA_{t+1}^i=\sum_{t=1}^T f_t^{\rm can}(x_t).
\label{eq:Jalg}
\end{equation}
Let $J_T^{\rm off}$ denote the cumulative weighted AoI of the offline dynamic optimum, the policy that knows $(C_t)_{t=1}^T$ in advance and minimizes $J_T$ subject to the same fluid AoI dynamics.

\begin{theorem}[Deadline-induced competitive ratio]
\label{thm:cr-deadline}
Under Assumptions~\ref{ass:causal}--\ref{ass:viability} and the canonical loss,
\begin{equation}
J_T^{\rm off}\ge TW,\qquad J_T^{\rm alg}\le T\sum_{i=1}^N w_iD_i,
\label{eq:cr-bounds}
\end{equation}
hence
\begin{equation}
\dfrac{J_T^{\rm alg}}{J_T^{\rm off}}\le\rho_D:=\dfrac{\sum_i w_iD_i}{W}.
\label{eq:rho-D}
\end{equation}
With margin $\xi$, $D_i$ may be replaced by $D_i-\xi_i$.
\end{theorem}

\begin{proof}
For any policy $\pi$ and slot $t$, $0\le(C_tx_t^\pi)_i\le 1$, so by induction $A_{t+1}^{i,\pi}\ge 1$ and $\sum_i w_iA_{t+1}^{i,\pi}\ge W$, hence $J_T^{\rm off}\ge TW$. Theorem~\ref{thm:main} gives $A_{t+1}^i\le D_i$, so $J_T^{\rm alg}\le T\sum_iw_iD_i$. Dividing yields~\eqref{eq:rho-D}.
\end{proof}

The ratio $\rho_D$ is data-only, in the sense that it is computable from $w$ and $D$ alone, without any knowledge of the channel realization or the loss sequence; this is exactly the regime in which classical adversarial AoI competitive ratios are usually stated, but \textcolor{black}{$\rho_D$ additionally respects the modeled-state feasibility scope of Theorem~\ref{thm:main}}, which is unavailable to soft-cost competitive analyses such as~\cite{banerjee2020adversarial,bhattacharjee2023competitive}.

\begin{theorem}[Regret-refined trace-wise ratio]
\label{thm:cr-trace}
Suppose the canonical loss is used and $\mathcal{K}_{1:T}\ne\varnothing$. Define the trace-wise static safe benchmark
\begin{equation}
J_T^{\rm stat}=\min_{u\in\mathcal{K}_{1:T}}\sum_{t=1}^T f_t^{\rm can}(u),
\label{eq:Jstat}
\end{equation}
and let $\alpha_T:=J_T^{\rm stat}/J_T^{\rm off}$. With $\eta=R/(G\sqrt{T})$, Algorithm~\ref{alg:main} satisfies
\begin{equation}
J_T^{\rm alg}\le J_T^{\rm stat}+RG\sqrt{T},
\label{eq:cr-trace-1}
\end{equation}
\begin{equation}
\dfrac{J_T^{\rm alg}}{J_T^{\rm off}}\le\alpha_T+\dfrac{RG}{W\sqrt{T}}.
\label{eq:cr-trace-2}
\end{equation}
\end{theorem}

\begin{proof}
Let $u_T^\star\in\arg\min_{u\in\mathcal{K}_{1:T}}\sum_t f_t^{\rm can}(u)$. Theorem~\ref{thm:main} with $\eta=R/(G\sqrt T)$ gives $J_T^{\rm alg}-J_T^{\rm stat}\le RG\sqrt T$, the first inequality. Dividing by $J_T^{\rm off}\ge TW$ from Theorem~\ref{thm:cr-deadline} gives the second.
\end{proof}

The deadline-induced ratio $\rho_D$ depends only on weights and deadlines, which is the tightest data-only ratio achievable for a hard-safety algorithm. The trace-wise ratio sharpens this whenever a single safe action does well across the trace, with an additive $O(1/\sqrt T)$ deviation from $\alpha_T$, and it is the only competitive bound we are aware of that is simultaneously refined by no-regret learning, certified for adversarial channels, and tight against the offline dynamic optimum on traces where $\alpha_T$ approaches one.

\subsection{Why Long-Term Virtual Queues Cannot Replace the Hard Shield}
\label{sec:why-shield}

\begin{proposition}[Sublinear cumulative violation $\not\Rightarrow$ zero stage-wise violation]
\label{prop:sublinear}
There exist trajectories of a long-term-constraint OCO algorithm with $\sum_{t=1}^T[g_t(x_t)]_+=O(1)$ and yet $[g_{t_0}(x_{t_0})]_+>0$ for some $t_0\in[T]$. In particular, sublinear or even bounded cumulative violation does not imply $(1,1)$-firm peak-AoI safety.
\end{proposition}

\begin{proof}
Consider any trajectory with $[g_1(x_1)]_+=1$ and $[g_t(x_t)]_+=0$ for all $t\ge 2$. Then $\sum_{t=1}^T[g_t(x_t)]_+=1=O(1)$ for every $T$, satisfying any sublinear cumulative-violation guarantee, but the slot-1 violation $[g_1(x_1)]_+>0$ produces $A_2^i>D_i$ for some sensor $i$, breaking peak-AoI safety and the corresponding $(1,1)$-firm requirement. Such trajectories are realized by drift-plus-penalty schedulers operating with an empty initial queue under an adversarial $C_1$ that makes $\mathcal K_1$ a tight singleton: the unconstrained gradient step lands outside $\mathcal K_1$, the queue absorbs the violation in subsequent slots, and the cumulative sum stays $O(1)$ even though the deadline is missed at $t=1$.
\end{proof}

A single missed deadline can already trip a circuit breaker, destabilize a closed-loop controller, or violate a $(1,1)$-firm requirement, so the hard projection in Algorithm~\ref{alg:main} is a structural necessity rather than a convenience: the queue $Q_t$ is repurposed as an a-posteriori certificate, while safety itself is conferred by the per-slot shield $\Pi_{\mathcal K_t}$. This shield-and-dual decomposition loses nothing in regret order while gaining everything in real-time semantics.

\section{Experiments}
\label{sec:experiments}

We evaluate OCO-PAoI-Hard on an adversarial shared-channel testbed targeting the three guarantees: \textcolor{black}{modeled fluid-state PAoI safety} (Thm.~\ref{thm:main}), $O(\sqrt T)$ regret (Thms.~\ref{thm:main},~\ref{thm:dyn}), and margin-safe robustness (Thm.~\ref{thm:margin-eps}). Default settings: $N=4$, $D=(6,7,8,9)$, $w=\mathbf 1$, $A_1=\mathbf 1$, $\eta_t=R/(G\sqrt t)$ with $R=\sqrt 2$ and $G$ from Lemma~\ref{lem:lip}; each number is averaged over $10$ seeds. Baselines: \emph{Vanilla OGD} (simplex-projected gradient without shield, same step size, so any safety gap is attributable to the missing shield); \emph{Long-term VQ} (drift-plus-penalty $\arg\min_x V f_t(x)+\sum_i Q_{t,i}g_{t,i}(x)$ with $V=\sqrt T$, the standard parameter achieving $O(\sqrt T)$ regret and $O(\sqrt T)$ cumulative violation); \emph{Greedy max-deficit} ($\arg\min_i(D_i-A_t^i)$); and \emph{Round-robin}. The trap channel fixes a safe column of $C_t$ with entries $0.22$ and a trap column in $\{0,1\}$ aligned with the canonical gradient, then biases the gradient by $\sum_iw_iD_i$ toward the trap, pulling any gradient-follower off the safe column; $\mathcal K_t$ remains non-empty at every slot.

\begin{figure}[!t]
\centering
\includegraphics[width=\linewidth]{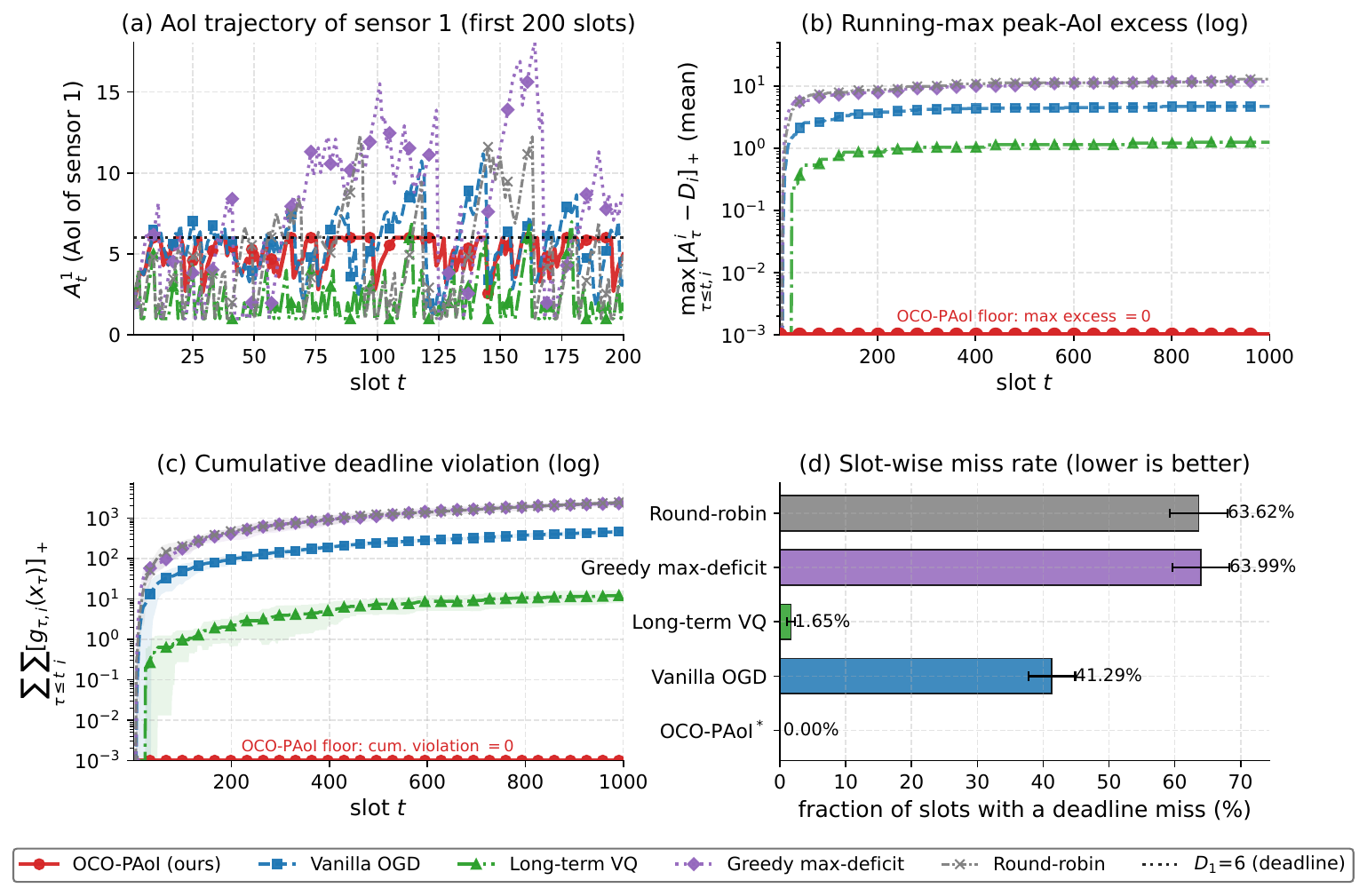}
\caption{\textcolor{black}{Modeled fluid-state peak-AoI safety} on the adversarial trap channel ($N\!=\!4$, $T\!=\!1000$, $10$ seeds, $D=(6,7,8,9)$). (a) AoI trajectory of sensor~1 during the first $200$ slots, with deadline $D_1\!=\!6$ (dotted); only OCO-PAoI-Hard (red) stays at or below it. (b) Running-max excess $\max_{\tau\le t,\,i}[A_\tau^i-D_i]_+$ (log scale; visible red floor is a plotting offset, true value $0$). (c) Cumulative deadline violation $\sum_{\tau\le t,\,i}[g_{\tau,i}(x_\tau)]_+$ (log scale). (d) Per-seed slot-wise miss rate.}
\label{fig:exp1}
\end{figure}

\begin{table}[t]
\caption{\textcolor{black}{Fluid-state PAoI safety} on the four-sensor adversarial trap channel ($T\!=\!1000$, $10$ seeds). ``Miss'' is the per-slot miss rate; ``Excess'' is $\max_{t,i}[A_t^i-D_i]_+$; ``CumViol'' is $\sum_{t,i}[g_{t,i}(x_t)]_+$; ``Cost'' is $\sum_t f_t^{\rm can}(x_t)$. Bold marks the only zero-violation entry. Long-term~VQ reaches the lowest cost only by accepting per-slot violations.}
\label{tab:safety-main}
\centering
\footnotesize
\renewcommand{\arraystretch}{1.12}
\setlength{\tabcolsep}{3pt}
\begin{tabular}{|l|c|c|c|c|}
\hline
Method & Miss & Excess & CumViol & Cost\\
\hline
Vanilla OGD               & $41.3\%$ & $4.70$  & $469.1$  & $20{,}703$\\
\hline
Long-term VQ              & $1.65\%$ & $1.26$  & $12.2$   & $\mathbf{7{,}950}$\\
\hline
Greedy max-deficit        & $64.0\%$ & $11.8$  & $2{,}358$ & $20{,}563$\\
\hline
Round-robin               & $63.6\%$ & $12.96$ & $2{,}469$ & $20{,}674$\\
\hline
\textbf{Ours}             & $\mathbf{0.0\%}$ & $\mathbf{1.3\text{e-}15}$ & $\mathbf{2.3\text{e-}14}$ & $20{,}075$\\
\hline
\end{tabular}
\end{table}

On $T=1000$ (Fig.~\ref{fig:exp1}, Table~\ref{tab:safety-main}), OCO-PAoI-Hard records \textcolor{black}{zero modeled fluid-state deadline violations} across all $10$ seeds, with maximum AoI excess $1.3\!\times\!10^{-15}$ (round-off) and exactly zero cumulative violation. The four baselines miss $1.65$--$64.0\%$ of slots with overshoots up to $12.96$. Long-term~VQ achieves a lower weighted-AoI cost ($7{,}950$ vs.\ $20{,}075$) only by accepting per-slot violations (Prop.~\ref{prop:sublinear}); among violation-free methods, ours has the lowest cost, beating Vanilla~OGD, Greedy max-deficit, and Round-robin by $3$ to $3.6\%$.

\begin{figure}[!t]
\centering
\includegraphics[width=\linewidth]{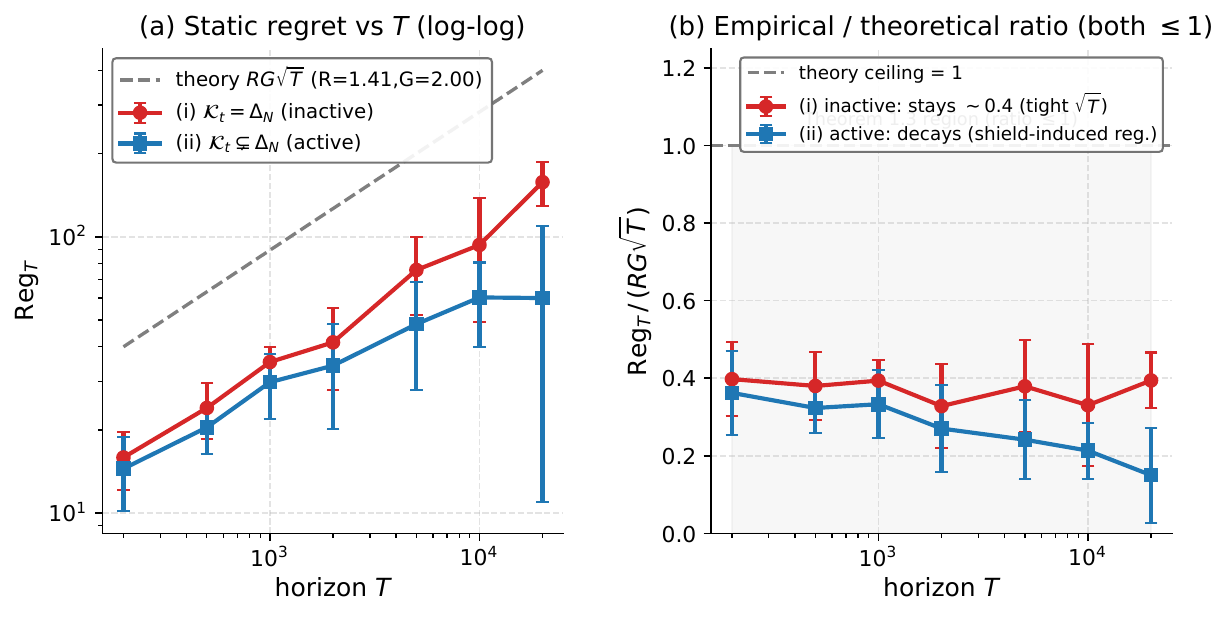}
\caption{Sublinear regret scaling under both inactive and active safe sets ($N\!=\!4$, $10$ seeds per $T$). (a) Static regret $\operatorname{Reg}_T$ vs.\ $T$ (log-log) against $RG\sqrt{T}$ (gray dashed); empirical slopes $0.484$ (inactive) and $0.324$ (active). (b) Normalized ratio $\operatorname{Reg}_T/(RG\sqrt{T})$ stays below the theoretical ceiling $1$ at every $T$.}
\label{fig:exp2}
\end{figure}

\emph{Sublinear regret.} We sweep $T\in\{200,500,10^3,2\!\times\!10^3,5\!\times\!10^3,10^4,2\!\times\!10^4\}$ in two regimes. The \emph{inactive} regime ($D_i=T+10$, so $\mathcal K_t=\Delta_N$) reduces the algorithm to OGD against adversarial Rademacher linear losses, isolating unconstrained-OCO behavior. The \emph{active} regime keeps $D=(6,7,8,9)$ and uses the static safe comparator $u^\star\in\arg\min_{u\in\mathcal K_{1:T}}\sum_t f_t^{\rm can}(u)$ computed by LP (\texttt{linprog}/HiGHS) on $\mathcal K_{1:T}$. Fig.~\ref{fig:exp2} confirms the $\sqrt T$ rate: empirical log-log slopes are $0.484$ (inactive) and $0.324$ (active), both below the $0.5$ ceiling of Thm.~\ref{thm:main}. The normalized ratio $\operatorname{Reg}_T/(RG\sqrt T)$ stays strictly below $1$, near $0.40$ in the inactive regime and decaying from $0.36$ to $0.15$ in the active regime, indicating the shield acts as an implicit regularizer; this matches the trace-wise refinement of Thm.~\ref{thm:cr-trace}: when a single safe action $u^\star$ already does well, the additive $RG/(W\sqrt T)$ term vanishes faster than the worst-case $O(1/\sqrt T)$.

\begin{figure}[!t]
\centering
\includegraphics[width=\linewidth]{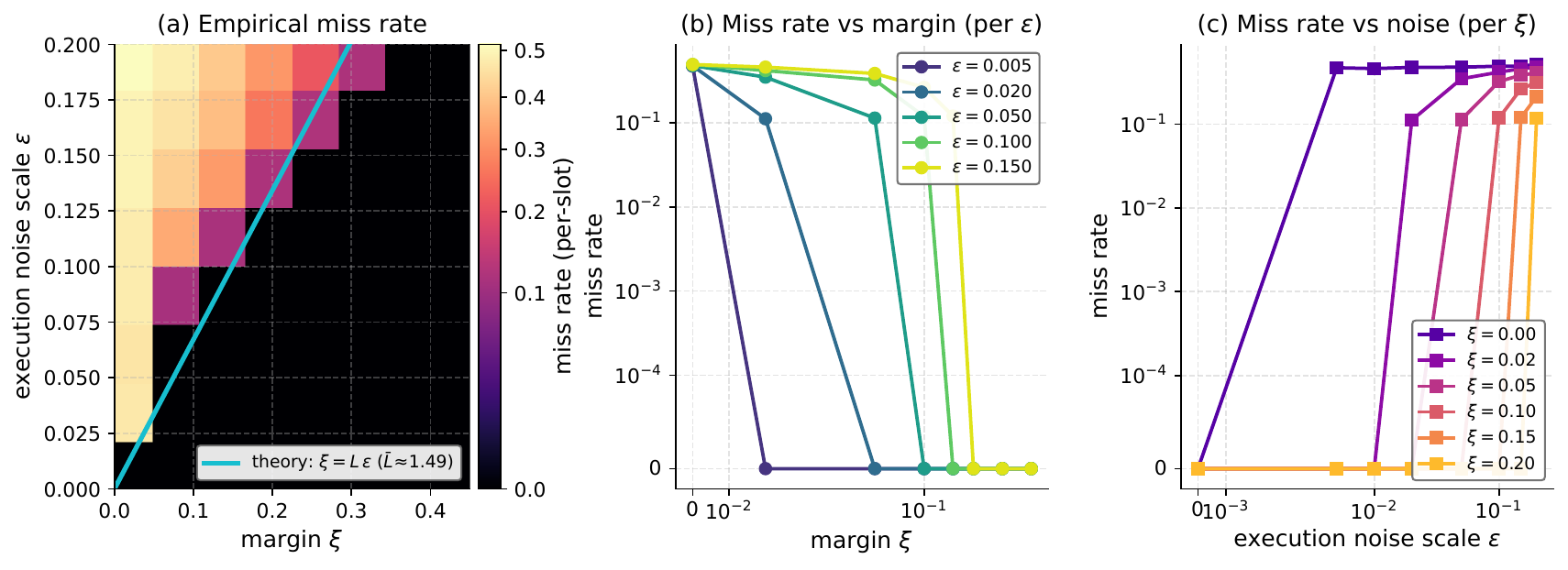}
\caption{Margin-safe robustness to execution noise ($N\!=\!4$, $T\!=\!400$, $5$ seeds per cell, $8\!\times\!8$ grid). (a) Empirical miss rate over $(\xi,\varepsilon)$; cyan line is $\xi=\bar L\,\varepsilon$ with $\bar L\approx 1.49$. Every $(\xi,\varepsilon)$ above the line has miss rate exactly zero. (b) Slices over margin; (c) slices over noise.}
\label{fig:exp3}
\end{figure}

\emph{Margin-safe robustness.} We compute $x_t\in\mathcal K_t^\xi$ but execute $\hat x_t=\Pi_{\Delta_N}(x_t+\varepsilon n_t)$ with $n_t$ unit-norm Gaussian, on a tight-channel construction (safe-column $0.15$, $D=(10,10,10,10)$). Sweeping $(\xi,\varepsilon)$ on an $8\!\times\!8$ grid ($64$ cells, $5$ seeds, $T=400$), Fig.~\ref{fig:exp3}(a) reveals a sharp zero-miss region right of $\xi=\bar L\varepsilon$ with $\bar L\approx 1.49$ the trace-averaged Lipschitz constant of Thm.~\ref{thm:margin-eps}. Without margin, even $\varepsilon=0.005$ drives miss rate to $47\%$; with $\xi\ge\bar L\varepsilon$ every cell achieves zero miss rate, matching the theoretical boundary to within one grid cell.

Three takeaways: (i) the hard projection is the only mechanism tested that \textcolor{black}{eliminates modeled fluid-state misses on every simulated trace}, while long-term VQ still violates $1.65\%$ of slots (Prop.~\ref{prop:sublinear}); (ii) the empirical regret slope sits well below $1/2$, suggesting the shield is an implicit regularizer in the active regime and that the constants in Thm.~\ref{thm:main} are not tight on benign traces; (iii) $\xi\ge L\varepsilon$ aligns with the empirical phase boundary to within one grid cell, validating Thm.~\ref{thm:margin-eps} as a deployable design rule that converts numerical and physical noise budgets directly into a margin shrinkage. The cost gap $20{,}075/7{,}950\approx 2.53$ is the price of \textcolor{black}{fluid-trajectory safety}, well within the deadline-induced bound $\rho_D=\sum_iw_iD_i/W=7.5$ of Thm.~\ref{thm:cr-deadline}.

\section{Conclusion}

We presented OCO-PAoI-Hard, a causal online scheduler that delivers \textcolor{black}{zero modeled-state peak-AoI violation} and $O(\sqrt{T})$ regret against adversarial channel and arrival sequences, \textcolor{black}{conditional on the encountered safe sets remaining non-empty}. \textcolor{black}{This guarantee concerns the modeled AoI state; packet-level safety requires stronger service assumptions.} The framework rests on a structural reduction that recasts the per-sensor peak-AoI deadline as an affine half-space constraint $C_t x \ge \theta_t$ on the fractional \textcolor{black}{resource-allocation vector}: one Euclidean projection onto the current-slot safe set acts as a \textcolor{black}{modeled-state} safety shield, while the gradient step preserves the $\sqrt{T}$ no-regret rate. Closed-form static and dynamic regret bounds, a matching minimax lower bound, margin and approximate-projection theorems, and two competitive-ratio statements together cover every operating mode, and experiments on a multi-sensor adversarial trap channel confirm \textcolor{black}{zero modeled fluid-state deadline violations}, sublinear empirical regret, and a margin-noise phase boundary matching the closed-form prediction to within one grid cell. Packet-level rounding, multi-cell coordination, and downstream control-loop stability are natural next directions.

% Acknowledgments are omitted for double-blind review and will be added in the camera-ready version.

% Keep newly added camera-ready references black together with their in-text discussion.
\makeatletter
\let\RTSS@bibitem\bibitem
\renewcommand{\bibitem}[1]{%
  \def\RTSS@thiskey{#1}%
  \def\RTSS@dynamickey{zhang2026dynamic}%
  \def\RTSS@multikey{zhang2026multiconstraint}%
  \def\RTSS@noisekey{zhang2026noiseadaptivehighprobabilityregretbounds}%
  \ifx\RTSS@thiskey\RTSS@dynamickey
    \color{black}%
  \else\ifx\RTSS@thiskey\RTSS@multikey
    \color{black}%
  \else\ifx\RTSS@thiskey\RTSS@noisekey
    \color{black}%
  \else
    \color{black}%
  \fi\fi\fi
  \RTSS@bibitem{#1}}
\makeatother

\bibliographystyle{IEEEtran}
\bibliography{references}

@IEEEtranBSTCTL{IEEEtranBSTcontrol,
  CTLdash_repeated_names = {no}
}

@inproceedings{kaul2012aoi,
  title={Real-time status: How often should one update?},
  author={Kaul, Sanjit and Yates, Roy and Gruteser, Marco},
  booktitle={2012 Proceedings IEEE INFOCOM},
  pages={2731--2735},
  year={2012},
  organization={IEEE}
}

@article{sun2017aoisurvey,
  title={Age of information: An introduction and survey},
  author={Yates, Roy D and Sun, Yin and Brown, D Richard and Kaul, Sanjit K and Modiano, Eytan and Ulukus, Sennur},
  journal={IEEE Journal on Selected Areas in Communications},
  volume={39},
  number={5},
  pages={1183--1210},
  year={2021},
  publisher={IEEE}
}

@standard{iec61850,
  title={Communication networks and systems for power utility automation---Part 5: Communication requirements for functions and device models},
  organization={International Electrotechnical Commission},
  number={{IEC 618{50}-5:2013+AMD1:2022 CSV}},
  year={2022},
  note={Edition 2.1},
  url={https://webstore.iec.ch/en/publication/75090}
}

@techreport{tr38913,
  author={{3GPP}},
  title={Study on Scenarios and Requirements for Next Generation Access Technologies},
  institution={3rd Generation Partnership Project},
  number={TR 38.913 V19.0.0},
  month=oct,
  year={2025},
  url={https://portal.3gpp.org/desktopmodules/Specifications/SpecificationDetails.aspx?specificationId=2996}
}

@inproceedings{hsu2018whittle,
  title={Age of information: Whittle index for scheduling stochastic arrivals},
  author={Hsu, Yu-Pin},
  booktitle={2018 IEEE International Symposium on Information Theory (ISIT)},
  pages={2634--2638},
  year={2018},
  organization={IEEE}
}

@article{maatouk2022whittle,
  title={On the optimality of the Whittle’s index policy for minimizing the age of information},
  author={Maatouk, Ali and Kriouile, Saad and Assaad, Mohamad and Ephremides, Anthony},
  journal={IEEE Transactions on Wireless Communications},
  volume={20},
  number={2},
  pages={1263--1277},
  year={2021},
  publisher={IEEE}
}

@article{tripathi2021whittle,
  title={A whittle index approach to minimizing functions of age of information},
  author={Tripathi, Vishrant and Modiano, Eytan},
  journal={IEEE/ACM Transactions on Networking},
  volume={32},
  number={6},
  pages={5144--5158},
  year={2024},
  publisher={IEEE}
}

@inproceedings{leng2019drl,
  title={Age of information minimization for wireless ad hoc networks: A deep reinforcement learning approach},
  author={Leng, Shiyang and Yener, Aylin},
  booktitle={2019 IEEE Global Communications Conference (GLOBECOM)},
  pages={1--6},
  year={2019},
  organization={IEEE}
}

@article{ceran2021drlaoi,
  title={A reinforcement learning approach to age of information in multi-user networks with HARQ},
  author={Ceran, Elif Tu{\u{g}}{\c{c}}e and G{\"u}nd{\"u}z, Deniz and Gy{\"o}rgy, Andr{\'a}s},
  journal={IEEE Journal on Selected Areas in Communications},
  volume={39},
  number={5},
  pages={1412--1426},
  year={2021},
  publisher={IEEE}
}

@article{abdelmagid2022drlsurvey,
  title={On the role of age of information in the Internet of Things},
  author={Abd-Elmagid, Mohamed A and Pappas, Nikolaos and Dhillon, Harpreet S},
  journal={IEEE Communications Magazine},
  volume={57},
  number={12},
  pages={72--77},
  year={2019},
  publisher={IEEE}
}

@inproceedings{banerjee2020adversarial,
  title={Fundamental limits of age-of-information in stationary and non-stationary environments},
  author={Banerjee, Subhankar and Bhattacharjee, Rajarshi and Sinha, Abhishek},
  booktitle={2020 IEEE International Symposium on Information Theory (ISIT)},
  pages={1741--1746},
  year={2020},
  organization={IEEE}
}

@article{bhattacharjee2023competitive,
  title={Competitive algorithms for minimizing the maximum age-of-information},
  author={Bhattacharjee, Rajarshi and Sinha, Abhishek},
  journal={ACM SIGMETRICS Performance Evaluation Review},
  volume={48},
  number={2},
  pages={6--8},
  year={2020},
  publisher={ACM New York, NY, USA}
}

@article{bhuyan2022adv,
  title={Optimizing age-of-information in adversarial and stochastic environments},
  author={Sinha, Abhishek and Bhattacharjee, Rajarshi},
  journal={IEEE Transactions on Information Theory},
  volume={68},
  number={10},
  pages={6860--6880},
  year={2022},
  publisher={IEEE}
}

@article{yu2017online,
  title={Online convex optimization with stochastic constraints},
  author={Yu, Hao and Neely, Michael and Wei, Xiaohan},
  journal={Advances in Neural Information Processing Systems},
  volume={30},
  year={2017}
}

@inproceedings{jenatton2016ococonstraint,
  title={Adaptive algorithms for online convex optimization with long-term constraints},
  author={Jenatton, Rodolphe and Huang, Jim and Archambeau, C{\'e}dric},
  booktitle={International Conference on Machine Learning},
  pages={402--411},
  year={2016},
  organization={PMLR}
}

@article{mahdavi2012trading,
  title={Trading regret for efficiency: online convex optimization with long term constraints},
  author={Mahdavi, Mehrdad and Jin, Rong and Yang, Tianbao},
  journal={The Journal of Machine Learning Research},
  volume={13},
  number={81},
  pages={2503--2528},
  year={2012},
  publisher={JMLR. org}
}

@article{bernat2001weaklyhard,
  title={Weakly hard real-time systems},
  author={Bernat, Guillem and Burns, Alan and Llamos{\'{\i}}, Albert},
  journal={IEEE transactions on Computers},
  volume={50},
  number={4},
  pages={308--321},
  year={2001},
  publisher={IEEE}
}

@article{hammadeh2017mkfirm,
  title={Weakly-hard real-time guarantees for earliest deadline first scheduling of independent tasks},
  author={Hammadeh, Zain AH and Quinton, Sophie and Ernst, Rolf},
  journal={ACM Transactions on Embedded Computing Systems (TECS)},
  volume={18},
  number={6},
  pages={1--25},
  year={2019},
  publisher={ACM New York, NY, USA}
}

@article{sun2023mkfirm,
  title={Weakly hard real-time model for control systems: A survey},
  author={Salamun, Karla and Pavi{\'c}, Ivan and D{\v{z}}apo, Hrvoje and {\v{C}}uljak, Ivana},
  journal={Sensors},
  volume={23},
  number={10},
  pages={4652},
  year={2023},
  publisher={MDPI}
}

@article{sun2019aoi,
  title={Update or wait: How to keep your data fresh},
  author={Sun, Yin and Uysal-Biyikoglu, Elif and Yates, Roy D and Koksal, C Emre and Shroff, Ness B},
  journal={IEEE Transactions on Information Theory},
  volume={63},
  number={11},
  pages={7492--7508},
  year={2017},
  publisher={IEEE}
}

@article{bedewy2021aoi,
  title={Optimal sampling and scheduling for timely status updates in multi-source networks},
  author={Bedewy, Ahmed M and Sun, Yin and Kompella, Sastry and Shroff, Ness B},
  journal={IEEE Transactions on Information Theory},
  volume={67},
  number={6},
  pages={4019--4034},
  year={2021},
  publisher={IEEE}
}

@article{kadota2018scheduling,
  title={Scheduling policies for minimizing age of information in broadcast wireless networks},
  author={Kadota, Igor and Sinha, Abhishek and Uysal-Biyikoglu, Elif and Singh, Rahul and Modiano, Eytan},
  journal={IEEE/ACM Transactions on Networking},
  volume={26},
  number={6},
  pages={2637--2650},
  year={2018},
  publisher={IEEE}
}

@inproceedings{zinkevich2003ogd,
  title={Online convex programming and generalized infinitesimal gradient ascent},
  author={Zinkevich, Martin},
  booktitle={Proceedings of the 20th international conference on machine learning (icml-03)},
  pages={928--936},
  year={2003}
}

@article{hazan2016oco,
  title={Introduction to online convex optimization},
  author={Hazan, Elad},
  journal={Foundations and Trends in Optimization},
  volume={2},
  number={3-4},
  pages={157--325},
  year={2016},
  publisher={now Publishers Inc.}
}

@inproceedings{yi2021oco,
  title={Regret and cumulative constraint violation analysis for online convex optimization with long term constraints},
  author={Yi, Xinlei and Li, Xiuxian and Yang, Tao and Xie, Lihua and Chai, Tianyou and Johansson, Karl},
  booktitle={International conference on machine learning},
  pages={11998--12008},
  year={2021},
  organization={PMLR}
}

@article{chen2017oco,
  title={An online convex optimization approach to proactive network resource allocation},
  author={Chen, Tianyi and Ling, Qing and Giannakis, Georgios B},
  journal={IEEE Transactions on Signal Processing},
  volume={65},
  number={24},
  pages={6350--6364},
  year={2017},
  publisher={IEEE}
}

@article{cao2019oco,
  title={Online convex optimization with time-varying constraints and bandit feedback},
  author={Cao, Xuanyu and Liu, K. J. Ray},
  journal={IEEE Transactions on automatic control},
  volume={64},
  number={7},
  pages={2665--2680},
  year={2019},
  publisher={IEEE}
}

@inproceedings{valls2020online,
  title={Online convex optimization with perturbed constraints: Optimal rates against stronger benchmarks},
  author={Valls, Victor and Iosifidis, George and Leith, Douglas and Tassiulas, Leandros},
  booktitle={International Conference on Artificial Intelligence and Statistics},
  pages={2885--2895},
  year={2020},
  organization={PMLR}
}

@article{davisburns2011realtime,
  title={A survey of hard real-time scheduling for multiprocessor systems},
  author={Davis, Robert I and Burns, Alan},
  journal={ACM computing surveys (CSUR)},
  volume={43},
  number={4},
  pages={35:1--35:44},
  year={2011},
  publisher={ACM New York, NY, USA}
}

@article{bennis2018urllc,
  title={Ultrareliable and low-latency wireless communication: Tail, risk, and scale},
  author={Bennis, Mehdi and Debbah, M{\'e}rouane and Poor, H Vincent},
  journal={Proceedings of the IEEE},
  volume={106},
  number={10},
  pages={1834--1853},
  year={2018},
  publisher={IEEE}
}

@inproceedings{achiam2017cpo,
  title={Constrained policy optimization},
  author={Achiam, Joshua and Held, David and Tamar, Aviv and Abbeel, Pieter},
  booktitle={International conference on machine learning},
  pages={22--31},
  year={2017},
  organization={Pmlr}
}

@inproceedings{alshiekh2018shielding,
  title={Safe reinforcement learning via shielding},
  author={Alshiekh, Mohammed and Bloem, Roderick and Ehlers, R{\"u}diger and K{\"o}nighofer, Bettina and Niekum, Scott and Topcu, Ufuk},
  booktitle={Proceedings of the AAAI conference on artificial intelligence},
  volume={32},
  number={1},
  year={2018}
}

@article{paternain2019safe,
  title={Constrained reinforcement learning has zero duality gap},
  author={Paternain, Santiago and Chamon, Luiz and Calvo-Fullana, Miguel and Ribeiro, Alejandro},
  journal={Advances in Neural Information Processing Systems},
  volume={32},
  year={2019}
}

@article{zhang2026dynamic,
  title={Dynamic Regret with Untrusted Decision Predictions via Heterogeneous Expert Aggregation},
  author={Wentao Zhang},
  journal={Transactions on Machine Learning Research},
  issn={2835-8856},
  year={2026},
  url={https://openreview.net/forum?id=LWsEyfdnp9}
}

@article{zhang2026multiconstraint,
  title={Multi-Constraint Online Convex Optimization with Adversarial Constraints},
  author={Wentao Zhang},
  journal={Transactions on Machine Learning Research},
  issn={2835-8856},
  year={2026},
  url={https://openreview.net/forum?id=3sLjLHCGzS}
}

@article{zhang2026noiseadaptivehighprobabilityregretbounds,
  title={Noise-Adaptive High-Probability Regret Bounds for Online Convex Optimization},
  author={Wentao Zhang and Yutong Zhang and Wentao Mo},
  journal={arXiv preprint arXiv:2606.08028},
  year={2026},
  note={Accepted to ECML-PKDD 2026},
  url={https://arxiv.org/abs/2606.08028}
}

\end{document}